\documentclass[11pt]{article}

\usepackage[margin=1.1in]{geometry}
\usepackage{amsmath,amssymb,amsfonts,amsthm,mathtools}
\usepackage{graphicx}
\usepackage{parskip}
\usepackage{bm}
\usepackage{booktabs,multirow}
\usepackage{array}
\usepackage[authoryear,round]{natbib}
\usepackage{enumitem}
\usepackage{xcolor}
\usepackage{microtype}
\usepackage{hyperref}
\usepackage{cleveref}
\usepackage{titlesec}
\usepackage{fancyhdr}
\usepackage{subcaption}
\usepackage{placeins}
\usepackage{needspace}
\usepackage{tikz,pgfplots}
\usepgfplotslibrary{groupplots}
\pgfplotsset{compat=1.18}
\usetikzlibrary{arrows.meta,positioning,calc,fit,backgrounds}
\usepackage{float}
\floatstyle{ruled}
\newfloat{algorithm}{tbp}{loa}
\floatname{algorithm}{Algorithm}
\newtheorem{proposition}{Proposition}
\newtheorem{corollary}{Corollary}
\newtheorem{lemma}{Lemma}

\newtheorem{assumption}{Assumption}

\hypersetup{
  pdftitle={Conformal Bayes under Continuous Label Shift: Sensitivity Analysis and the Limits of Exact Validity},
  pdfauthor={Seungjin Choi},
  colorlinks=true,
  linkcolor=blue!70!black,
  citecolor=blue!50!black,
  urlcolor=blue!60!black
}

\providecommand{\E}{\mathbb{E}}
\providecommand{\Prob}{\mathbb{P}}
\newcommand{\Dtr}{\mathcal{D}_{\mathrm{tr}}}
\newcommand{\Dcal}{\mathcal{D}_{\mathrm{cal}}}
\newcommand{\HH}{\mathcal{H}}

\newcommand{\CC}{\mathcal{C}}
\newcommand{\R}{\mathbb{R}}
\newcommand{\1}{\mathbf{1}}
\newcommand{\N}{\mathcal{N}}
\newcommand{\mub}{\widehat{\mu}}
\newcommand{\sigb}{\widehat{\sigma}}

\newcommand{\ms}{m_s}
\newcommand{\vs}{v_s}

\titleformat{\section}{\large\bfseries\color{blue!60!black}}{{\thesection}}{1em}{}[\titlerule]
\titleformat{\subsection}{\normalsize\bfseries\color{blue!40!black}}{{\thesubsection}}{1em}{}
\titleformat{\subsubsection}{\normalsize\itshape\bfseries}{{\thesubsubsection}}{1em}{}

\begin{document}

% -------------------------------------------------------
% Title block
% -------------------------------------------------------
\begin{center}
  {\LARGE\bfseries \textsf{Conformal Bayes under Continuous Label Shift: \\ [5pt]
  Sensitivity Analysis and the Limits of Exact Validity}}\\[2em]
%  {\large\itshape  Seungjin Choi}\\[0.5em]
  {\large  Seungjin Choi}\\[1em]
  {\normalsize CROID Research and aSSIST University, Seoul, Korea}
  \end{center}

\vspace{0.5em}
\noindent\rule{\linewidth}{1.5pt}
\vspace{0.5em}

\begin{abstract}
Conformal Bayes combines Bayesian posterior predictive scores with conformal
calibration. Under continuous label shift, both the score and calibration
weights depend on an unknown response density ratio. Existing methods
typically estimate a single exponential tilt from pseudo-labels or predictive
samples and plug it into calibration. We propose Joint Tilt-Sensitivity
Conformal Bayes (JTS-CB), which performs sensitivity analysis over a
prespecified family of tilts. Each tilt defines a coupled score-weight pair.
The split procedure, JTS-SCB, solves joint min-max calibration by taking the
largest tilt-specific threshold from weighted intercept-only quantile
regression. JTS-Union retains individual thresholds, yielding prediction sets
contained in those of JTS-SCB. Both use calibration-only weights by design,
without a general coverage guarantee. We also study a candidate-weighted exact
counterpart. Any nonzero linear exponential tilt makes its prediction set
unbounded, whereas compact negative-quadratic families yield bounded exact
unions under Gaussian scores. For the practical procedures, a localized
correction with fallback gives coverage at least nominal minus a target
tail-mass bound when the family contains the true tilt, but can substantially
increase width. Experiments on controlled shifts of three regression datasets
and on synthetic data show that plug-in predictive sampling can approach
oracle coverage when the shift is well identified, but remains vulnerable to
predictive bias and unstable tilt estimation. Sensitivity analysis can
mitigate these failures in the tested settings, at the cost of wider
prediction sets. Coupling scores and weights can yield narrower sets with
higher coverage than varying weights alone, although this advantage is not
uniform across shifts. Excluding the true tilt from the sensitivity family
can still cause undercoverage.
\end{abstract}

\section{Introduction}
\label{sec:intro}

Conformal prediction provides a general framework for predictive uncertainty
quantification with finite-sample marginal coverage under exchangeability
\citep{VovkV2005book,LeiJ2018jasa,ShaferG2008jmlr,
AngelopoulosAN2023ftml}. In split conformal prediction, a predictive model is
fitted on training data, a nonconformity score is evaluated on a held-out
calibration sample, and a score quantile determines the prediction set. The
predictive model controls efficiency, while the calibration rank supplies
validity.

Let $X\in\mathcal X$ denote the input and $Y\in\R$ a continuous response.
Given source training data $\Dtr$, write $p_s(y\mid x,\Dtr)$ for the fitted
Bayesian posterior predictive density at response $y$ and input $x$.
Conformal Bayes uses its negative logarithm as the nonconformity score,
\[
 S_s(x,y)=-\log p_s(y\mid x,\Dtr),
\]
assigning small values to responses that the fitted model regards
as plausible \citep{MelluishT2001ecml,WassermanL2011ss,FongE2021neurips}.

Distribution shift can undermine the coverage guarantee provided by conformal calibration. 
Source calibration scores need not represent the target population, 
so prediction sets based on the usual conformal quantile may fail to achieve nominal target coverage.
Let $P_s$ and $P_t$
denote the source and target laws of $(X,Y)$, with response marginals
$P_s^Y$ and $P_t^Y$ and conditional input densities $p_s(x\mid y)$ and
$p_t(x\mid y)$. We assume that $P_t^Y$ is absolutely continuous with
respect to $P_s^Y$. Continuous label shift preserves the conditional input
distribution while allowing the response marginal to change,
\begin{equation}
 p_s(x\mid y)=p_t(x\mid y),\qquad
 w(y)=\frac{dP_t^Y}{dP_s^Y}(y),
 \label{eq:label_shift}
\end{equation}
where $w$ is the target-to-source response density ratio. This is the
continuous-response analogue of target shift in domain adaptation
\citep{ZhangK2013icml}. Under this invariance, $w(y)$ is also the joint
density ratio. The calibration adjustment therefore depends on the
response, which is unobserved at a new target input.

For predictive-tilting Conformal Bayes, the ratio enters twice. An admissible
candidate ratio $w$ defines a tilted posterior predictive density $p_w$ through
\begin{equation}
 p_w(y\mid x,\Dtr)
 =\frac{p_s(y\mid x,\Dtr)w(y)}
 {\int p_s(u\mid x,\Dtr)w(u)\,du}.
 \label{eq:predictive_tilt_intro}
\end{equation}
For held-out source calibration observations $(X_i,Y_i)$, $i=1,\ldots,n$,
the same ratio supplies the importance weights $w(Y_i)$ used to calibrate
the resulting scores. We propagate shift uncertainty through these two roles
using the same tilt. The identity for the true
conditional densities follows from label shift. Applied to a fitted
predictive density, it defines a working model whose accuracy affects the
geometry and efficiency of the prediction set.

Recent continuous-label-shift Conformal Bayes methods estimate one tilt
from target pseudo-labels or predictive samples and plug it into calibration
\citep{Choi2026eiml,LeeHS2026eiml,Choi2026arxiv_scbc,Choi2026copa}.
These procedures can be efficient when the predictive model is reliable and
the shift is well identified. Their vulnerability lies in the auxiliary
target-label estimation problem. Uncertainty or bias in pseudo-responses
becomes uncertainty or bias in the density ratio, and a tail-sensitive
parameter can amplify that error.

Our starting point is instead a \emph{joint min--max calibration problem}.
Suppose an analyst can specify a family of plausible shifts more credibly
than a single shift estimate. For every candidate tilt, the same parameter
generates a Bayesian score and its calibration weights. Weighted quantile
regression then gives the smallest threshold meeting that tilt's empirical
calibration requirement. The outer optimization selects the largest of
these quantile solutions. Equivalently, it chooses the smallest common
threshold for which the worst weighted empirical calibration shortfall
over the family is nonpositive. Shift uncertainty is thereby incorporated
into the calibration decision itself.

We call this principle \emph{Joint Tilt-Sensitivity Conformal Bayes}
(JTS-CB) and its split procedure \emph{JTS-SCB}. Let
$\phi\colon\R\to\R^d$ be a chosen vector of shift features and
$\mathcal B\subset\R^d$ a prespecified set of candidate tilt parameters
$\boldsymbol\beta$. With $A_s(\boldsymbol\beta)$ denoting the source-marginal
log normalizer, we represent the response ratio by the exponential family
\begin{equation}
 w(y;\boldsymbol\beta)
 =\exp\{\boldsymbol\beta^\top\phi(y)-A_s(\boldsymbol\beta)\},
 \qquad \boldsymbol\beta\in\mathcal B.
 \label{eq:exp_tilt_intro}
\end{equation}
The features $\phi$ specify the form of shift, while $\mathcal B$ expresses
uncertainty about its parameters. The linear family
uses $\phi(y)=y$. The quadratic family uses $\phi(y)=(y,y^2)^\top$ and
includes linear shifts at $\beta_2=0$. This structure makes the uncertainty
class interpretable while allowing it to affect both calibration weights
and predictive geometry.

The optimization has two distinct roles. The supremum over quantile
solutions accounts for the largest calibration requirement, whereas a
subsequent union over score sublevel sets accounts for the different
predictive geometries. Keeping only the tilt that maximizes the threshold
would generally miss responses admitted by another tilt. JTS-SCB uses one
common threshold throughout that union. Its refinement, JTS-Union, keeps
each tilt's own quantile and is contained in the common-threshold envelope.
For Gaussian predictive densities and linear tilts, JTS-SCB has an explicit
interval form that separates the common radius from the spread of the
candidate predictive centers.

Sensitivity analysis does not identify the realized target shift. A
broader family can protect against more plausible shifts, but it also
increases width. When tilted predictive sampling reliably estimates the
shift, a plug-in can closely match the known-tilt benchmark and be more
efficient. A prespecified family is useful when limited target information
or predictive bias makes that estimate unreliable. The analyst still needs
a defensible family, and enlarging it cannot repair a violation of the
label-shift assumption.

Both practical procedures use calibration-only weights by design, so their
quantiles omit the candidate response's own importance weight. They therefore
do not inherit a general nominal coverage guarantee. We use exact weighted conformal
inversion as a benchmark for understanding that limitation. A large
candidate weight can force acceptance, making exact sets unbounded for
nonzero linear tilts when inverted over the response line. Compact strictly
negative-quadratic families avoid this tail behavior under Gaussian scores.
A localized candidate-weight bound also gives a conservative correction of
the same quantile problems, with a fallback and a target tail-mass penalty.
These results qualify the method's interpretation and identify the cost of
additional validity requirements.

The contribution is organized around three questions.
\begin{enumerate}[leftmargin=*,itemsep=3pt]
\item \emph{How should an uncertain tilt enter calibration?}
We formulate the joint optimization over coupled score--weight pairs,
relate it to weighted quantile regression and robust feasibility, and
characterize the resulting sensitivity envelopes. A fixed-score sensitivity
baseline isolates what predictive tilting adds beyond weight uncertainty.
\item \emph{What does that empirical calibration requirement guarantee?}
We distinguish the practical sets from exact candidate-weighted inversion,
describe the relevant tail behavior, and give a localized correction with
explicit containment and coverage conditions.
\item \emph{When is the width cost of sensitivity analysis justified?}
Controlled shifts of three real regression datasets provide the principal
comparison, including a deliberately misspecified budget. Two synthetic
studies isolate predictive bias and assess quadratic tilt estimation with
limited target inputs under common parameter bounds.
\end{enumerate}
The contribution is the family-level calibration construction and its
analysis. Fixed-tilt score--weight coupling, the quantile identity, and
weighted exchangeability are established ingredients. JTS-CB combines them
by optimizing the selected quantile over a structured family of coupled
pairs and carrying that family into prediction. The supporting results
identify the resulting geometry and validity limitations, rather than a
new general quantile or coverage theorem. The empirical comparisons assess
this construction without claiming that it dominates reliable shift estimation.

Section~\ref{sec:related} places the construction in the context of related
work. Section~\ref{sec:setup} develops the fixed-tilt ingredients, and
Section~\ref{sec:method} presents the joint calibration problem and prediction
sets. Sections~\ref{sec:validity} and~\ref{sec:experiments} examine validity
and empirical performance, followed by discussion in
Section~\ref{sec:discussion}.

\section{Related work}
\label{sec:related}

Weighted conformal prediction extends the calibration rank argument to
known source-to-target likelihood ratios
\citep{TibshiraniR2019neurips}. Its general weighted-exchangeability result
permits a ratio that depends on the response, and
\citet{PodkopaevA2021uai} develop the classification label-shift case with
candidate-label-dependent weights. Work beyond exchangeability also treats
more general departures from the calibration distribution
\citep{BarberRF2023aos}. Our setting is structured by the invariance in
\eqref{eq:label_shift}, and the unknown object is a continuous
response-density ratio.

For classification label shift, \citet{SaerensM2002neco} adjust class
probabilities by likelihood-based iteration, while \citet{LiptonZ2018icml}
estimate the shift through black-box predictions.
\citet{AlexandariAM2020icml,GargS2020neurips} clarify the role of predictive
calibration and the relation between likelihood and moment-based estimation.
These precedents motivate the plug-in comparison, while our response is continuous.

Bayesian predictive scores and exponential density-ratio models are
established ingredients
\citep{MelluishT2001ecml,FongE2021neurips,QinJ1998biometrika}.
The fixed-tilt Conformal Bayes construction and its Gaussian form follow
\citet{Choi2026eiml}. Related plug-in developments use pseudo-labels or
predictive samples to estimate the ratio
\citep{LeeHS2026eiml,Choi2026arxiv_scbc,Choi2026copa}.
For continuous target shift, \citet{YiM2025arxiv} estimate the ratio by
matching source and target input distributions in a finite spline space,
then construct prediction sets using a weight-adjusted conditional density.
JTS-CB instead takes a tilt family as input and propagates its members
through calibration without selecting a single fitted ratio.

The pinball-loss characterization of quantiles is classical
\citep{KoenkerR2005book}. \citet[Section~2.1]{GibbsI2025jrsssb} interpret
conformal calibration as quantile regression, augmenting the calibration
problem with the candidate score to obtain validity. Their broader
construction addresses coverage over classes of input-distribution shifts. Our inner
problem is weighted intercept-only quantile regression of the calibration
scores, not quantile regression of the response on the original inputs.
The same tilt changes both those scores and the regression weights. The
outer problem selects their largest lower-quantile solution. It is not a
claim that minimizing a worst-case pinball-loss value solves the same
problem, or that the extrema can be interchanged.

Sensitivity analysis asks how conclusions vary over an analyst-specified
class of departures from an identifying or sampling model
\citep{RosenbaumPR2002book,TanZ2006jasa,ZhaoQ2019jrsssb}.
Conformal sensitivity methods for unmeasured confounding similarly vary
importance weights over a class
\citep{JinY2023pnas,YinM2024jasa}. In particular, upper bounds on candidate
weights and worst-case weighted quantiles are already used by
\citet{JinY2023pnas}. Our correction applies that device within the coupled
tilt family, localizes the bound to a response range, and tracks the
resulting containment and tail penalty. It does not introduce a new
weighted-calibration principle.

Distributionally robust conformal methods also protect against classes of
shifts, including divergence and transport-based ambiguity sets
\citep{CauchoisM2024jasa,AiJ2024icml,AolariteiL2025neurips,XuR2025iclr}.
Here the class is a low-dimensional family of response-marginal tilts,
and each member determines a predictive density as well as a weight.
The family is therefore both a sensitivity specification and a restriction
on how score and weight uncertainty may be combined.

\section{From Conformal Bayes to predictive tilting}
\label{sec:setup}

We first specify the fixed-tilt ingredients needed for the joint calibration
problem. Let $\Dtr$ be source training data and
$\Dcal=\{(X_i,Y_i)\}_{i=1}^n$ a held-out source calibration sample, with
$n\ge1$ and $\alpha\in(0,1)$. A fitted Bayesian model supplies
$p_s(y\mid x,\Dtr)$. All predictive tilts below are assumed normalizable,
with finite scores and positive finite weights at the evaluated points.
The conditional sampling requirements for the coverage statements are given
in Section~\ref{sec:validity}.

To recall the ordinary split-conformal construction, fix a score $S$ using
only $\Dtr$ and order its calibration values as
$S_{(1)}\le\cdots\le S_{(n)}$, where $S_i=S(X_i,Y_i)$.
With the convention $S_{(n+1)}=+\infty$, define
\begin{equation}
 \begin{aligned}
 k&=\left\lceil(n+1)(1-\alpha)\right\rceil,
 &\widehat q_{\rm split}&=S_{(k)},\\
 \CC_{\rm split}(x)&=\{y:S(x,y)\le\widehat q_{\rm split}\}.
 \end{aligned}
 \label{eq:ordinary_split_conformal}
\end{equation}
If the calibration and test observations are exchangeable conditional on
$\Dtr$, this set has marginal coverage at least $1-\alpha$
\citep{LeiJ2018jasa}. Taking $S=S_s$ gives ordinary split Conformal Bayes.
The $(n+1)$ rank adjustment accounts for the test observation. It is absent
from the calibration-only quantiles in Section~\ref{sec:quantile_regression},
even when all weights are equal. Under label shift, calibration and test
observations need not be exchangeable, motivating the weighting and
predictive tilting below.

\subsection{The two effects of a fixed tilt}
\label{sec:fixed_tilt}

Write the normalizer in \eqref{eq:exp_tilt_intro} as
\begin{equation}
 A_s(\boldsymbol\beta)
 =\log\E_{P_s^Y}\exp\{\boldsymbol\beta^\top\phi(Y)\}.
 \label{eq:marginal_normalizer}
\end{equation}
For a candidate $\boldsymbol\beta$, predictive tilting gives
\begin{equation}
 p_{\boldsymbol\beta}(y\mid x,\Dtr)
 =\frac{p_s(y\mid x,\Dtr)
             e^{\boldsymbol\beta^\top\phi(y)}}
 {Z_x(\boldsymbol\beta)},\qquad
 Z_x(\boldsymbol\beta)
 =\int p_s(u\mid x,\Dtr)e^{\boldsymbol\beta^\top\phi(u)}\,du.
 \label{eq:posterior_tilt}
\end{equation}
The source-marginal factor $e^{-A_s(\boldsymbol\beta)}$ cancels from this
conditional density. It also cancels from normalized calibration weights
and from an exact weighted conformal $p$-value when the same scaling is
used for calibration and candidate weights.

The coupled calibration pair is
\begin{equation}
 \begin{split}
 S_i(\boldsymbol\beta)
 &=-\log p_{\boldsymbol\beta}(Y_i\mid X_i,\Dtr)\\
 &=-\log p_s(Y_i\mid X_i,\Dtr)
       -\boldsymbol\beta^\top\phi(Y_i)+\log Z_{X_i}(\boldsymbol\beta),\\
 W_i(\boldsymbol\beta)&=w(Y_i;\boldsymbol\beta).
 \end{split}
 \label{eq:coupled_pair}
\end{equation}
This expression separates the tilt's two effects. It changes how surprising
a calibration response is under the fitted predictive density, and it
changes how much that observation contributes to the target calibration
distribution. The predictive normalizer $Z_{X_i}(\boldsymbol\beta)$
generally varies with the input and is part of the score. It cannot be
dropped in the way that the common weight normalizer can.

The predictive tilt is a modeling choice for efficiency. Exact weighted
conformal validity requires the correct likelihood ratio and a fixed score
function, not correctness of the Bayesian model or a score induced by that
ratio. We nevertheless use the coupled construction because each candidate
then represents one internally consistent shift of both prediction and
calibration.

\subsection{Linear and quadratic Gaussian tilts}
\label{sec:gaussian}

Suppose the source posterior predictive density is
\begin{equation}
 p_s(y\mid x,\Dtr)=\N\{\mub(x),\sigb^2(x)\},\qquad
 0<\sigb^2(x)<\infty.
 \label{eq:gaussian_source}
\end{equation}
For a scalar linear tilt, completing the square gives
\begin{align}
 p_\beta(y\mid x,\Dtr)
 &=\N\{\mub(x)+\beta\sigb^2(x),\sigb^2(x)\},
 \label{eq:gaussian_tilt}\\
 S_\beta(x,y)
 &=\frac{\{y-\mub(x)-\beta\sigb^2(x)\}^2}{2\sigb^2(x)}
   +\frac12\log\{2\pi\sigb^2(x)\}.
 \label{eq:gaussian_score}
\end{align}
Thus the same $\beta$ that exponentially reweights labels moves the
predictive center by $\beta\sigb^2(x)$. Inputs with greater predictive
variance experience a larger displacement. The fixed-tilt construction and
this Gaussian specialization are those of \citet{Choi2026eiml}.

For $\phi(y)=(y,y^2)^\top$, the tilted predictive remains Gaussian, with
\begin{equation}
 \sigb^2_{\boldsymbol\beta}(x)
 =\frac{\sigb^2(x)}{1-2\beta_2\sigb^2(x)},\qquad
 \mub_{\boldsymbol\beta}(x)
 =\frac{\mub(x)+\beta_1\sigb^2(x)}{1-2\beta_2\sigb^2(x)}.
 \label{eq:quadratic_predictive}
\end{equation}
Its score is the negative log density with these mean and variance
parameters. The condition $1-2\beta_2\sigb^2(x)>0$ is necessary for this
predictive tilt to exist. The source-marginal tilt must also be
normalizable. For a Gaussian source label marginal with variance $\vs^2$,
that requires $\beta_2<1/(2\vs^2)$. A compact family must stay within the
admissible parameter region at the inputs where it is used.

Figure~\ref{fig:tilt_marginals} illustrates the corresponding marginal
changes. A linear tilt shifts a Gaussian location without changing its
variance. A quadratic tilt can also change spread. The sign of $\beta_2$
will matter for the exact candidate-weighted benchmark, but both families
enter the same joint calibration problem.

\begin{figure}[t]
\centering
\begin{tikzpicture}
\begin{groupplot}[
    group style={group size=3 by 1,horizontal sep=0.9cm},
    width=0.30\textwidth,
    height=4.5cm,
    xlabel={$y$},
    ymin=0,
    ymax=0.62,
    axis lines=left,
    samples=300,
    tick label style={font=\small},
    label style={font=\small},
    title style={font=\small\bfseries},
    legend style={draw=none,fill=none,font=\scriptsize,cells={anchor=west},
      at={(0.5,0.98)},anchor=north}
]
\nextgroupplot[
    xmin=-4.5,xmax=4.5,domain=-4.5:4.5,
    ylabel={density},title={(a) Source marginal}]
\addplot[blue!65!black,thick]
{1/sqrt(2*pi) * exp(-x^2/2)};
\addlegendentry{$p_s(y)=\mathcal N(0,1)$}

\nextgroupplot[
    xmin=-4.5,xmax=5.0,domain=-4.5:5.0,
    title={(b) Linear tilt}]
\addplot[blue!65!black,thick]
{1/sqrt(2*pi) * exp(-x^2/2)};
\addlegendentry{source}
\addplot[orange!80!black,thick,dashed]
{1/sqrt(2*pi) * exp(-(x-1)^2/2)};
\addlegendentry{tilted $\beta=1$}

\nextgroupplot[
    xmin=-5.0,xmax=6.0,domain=-5.0:6.0,
    title={(c) Quadratic tilt}]
\addplot[blue!65!black,thick]
{1/sqrt(2*pi) * exp(-x^2/2)};
\addlegendentry{source}
\addplot[orange!80!black,thick,dashed]
{1/sqrt(2*pi*(1/0.7)) * exp(-(x-(0.5/0.7))^2/(2*(1/0.7)))};
\addlegendentry{tilted $(0.5,0.15)$}
\end{groupplot}
\end{tikzpicture}
\caption{Illustration of exponential-family response-marginal tilting from a
standard Gaussian source.  (a) The source marginal is $p_s(y)=\mathcal N(0,1)$.
(b) The linear tilt $w(y)\propto\exp(\beta y)$ with $\beta=1$ shifts the mean
while preserving the variance.  (c) The quadratic tilt
$w(y)\propto\exp(\beta_1y+\beta_2y^2)$ with
$(\beta_1,\beta_2)=(0.5,0.15)$ changes both location and spread.  The linear
model is nested at $\beta_2=0$.}
\label{fig:tilt_marginals}
\end{figure}

\subsection{Specifying the sensitivity family}
\label{sec:budget}

The family $\mathcal B$ is an input to the analysis. It may be chosen from
scientific knowledge or training and independent reference data, but the
validity statements below require it to be fixed before seeing the
calibration and test draws. A range chosen after examining calibration
coverage is not covered by those statements.

For a Gaussian source label marginal $P_s^Y=\N(\ms,\vs^2)$, a scalar tilt
has normalized ratio
\begin{equation}
 w(y;\beta)
 =\exp\{\beta(y-\ms)-\tfrac12\beta^2\vs^2\},
 \label{eq:gaussian_weight}
\end{equation}
and shifts the label mean by $\beta\vs^2$. A bound $\kappa$ on the absolute
mean displacement therefore corresponds to
\begin{equation}
 \mathcal B(\kappa)
 =[-\kappa/\vs^2,\kappa/\vs^2].
 \label{eq:linear_budget}
\end{equation}
A one-sided interval is appropriate when the shift direction is supplied.
The interpretation in terms of mean displacement uses the Gaussian
marginal model. Calibration with weights proportional to $e^{\beta y}$
does not require that model, because its normalizer cancels.

A quadratic box allows location and dispersion uncertainty to be considered
together. Its upper bound on $\beta_2$ must respect normalizability.
The linear subfamily lies on $\beta_2=0$, so a single box may include
both forms without selecting a model first. Such a box may also be costly
in width. Containing the true tilt is a substantive assumption, not a
consequence of choosing a flexible family. In the experiments the budgets
are supplied directly in parameter space.

\subsection{The plug-in alternative}
\label{sec:pseudo}

A plug-in procedure constructs surrogate responses at unlabeled target
inputs and estimates $\boldsymbol\beta$ before calibration. Point
pseudo-labeling uses the fitted predictive means. Source predictive
sampling (SPS) draws from the untilted fitted density and retains its
predictive variation. Tilted predictive sampling (TPS) draws from the
currently tilted density, refits the ratio, and iterates toward a
self-consistent fit. Analytic predictive expectations can replace Monte
Carlo sampling when available.

For pseudo-responses $\widetilde y_{1:m}$, moment fitting in the exponential
family uses
\begin{equation}
 \E_{\boldsymbol\beta}\{\phi(Y)\}
 =\frac1m\sum_{j=1}^m\phi(\widetilde y_j),\qquad
 \E_{\boldsymbol\beta}\{\phi(Y)\}
 =\frac{\E_{P_s^Y}[\phi(Y)e^{\boldsymbol\beta^\top\phi(Y)}]}
 {\E_{P_s^Y}[e^{\boldsymbol\beta^\top\phi(Y)}]}.
 \label{eq:moment_matching}
\end{equation}
These are the stationarity equations of a convex exponential-family
fitting objective. A finite solution requires attainable pseudo-moments.
An alternative used in the quadratic experiment is logistic
regression of domain membership on $(y,y^2)$, a standard exponential
ratio estimator \citep{QinJ1998biometrika}.

TPS is a strong comparator because it feeds the proposed shift back into
the predictive distribution. A bootstrap range of TPS estimates can
propagate sampling uncertainty as well. Neither operation necessarily
corrects systematic predictive bias. JTS-SCB changes this estimation step
by taking $\mathcal B$ as given and calibrating across it. It needs no
target pseudo-responses once that family has been specified.

\section{Joint Tilt-Sensitivity Conformal Bayes}
\label{sec:method}

The method begins with a joint min--max calibration problem. A candidate
tilt acts on both the responses to be quantile-calibrated, through their
Bayesian scores, and the weights assigned to those responses. We first
derive the common threshold through weighted quantile regression and
robust feasibility, then construct the prediction sets.

\subsection{Joint min--max calibration through weighted quantile regression}
\label{sec:quantile_regression}

Fix a nonempty family $\mathcal B$, $n\ge1$, and $\tau=1-\alpha\in(0,1)$.
For every tilt assume the calibration scores are finite and the weights
are finite and strictly positive. With the coupled pairs from
\eqref{eq:coupled_pair}, define
\begin{equation}
 \overline W_i(\boldsymbol\beta)
 =\frac{W_i(\boldsymbol\beta)}{\sum_j W_j(\boldsymbol\beta)},
 \qquad
 \widehat F_{\boldsymbol\beta}(q)
 =\sum_{i=1}^n\overline W_i(\boldsymbol\beta)
      \1\{S_i(\boldsymbol\beta)\le q\}.
 \label{eq:jts_weighted_cdf}
\end{equation}
The weighted pinball objective is
\begin{equation}
 L_{\boldsymbol\beta}(q)
 =\sum_{i=1}^n\overline W_i(\boldsymbol\beta)
      \rho_\tau\{S_i(\boldsymbol\beta)-q\},
 \qquad
 \rho_\tau(u)=\tau u_+ +(1-\tau)(-u)_+,
 \label{eq:weighted_pinball_loss}
\end{equation}
where $u_+=\max(u,0)$. For fixed $\boldsymbol\beta$, this is weighted
intercept-only quantile regression of the calibration scores on a constant.
It is convex in the scalar threshold $q$. The lower weighted empirical quantile is
\begin{equation}
 \widehat q_{\boldsymbol\beta}
 =\inf\{q\in\R\mid\widehat F_{\boldsymbol\beta}(q)\ge\tau\}.
 \label{eq:practical_quantile}
\end{equation}
We write $\widehat q_{\boldsymbol\beta}(a)$ when the miscoverage level
$a$ replaces $\alpha$. The following weighted form of the classical
pinball-loss characterization \citep{KoenkerR2005book} fixes ties and
nonunique minimizers. Weighting changes the empirical distribution in the
usual subgradient argument, not the characterization itself.

\begin{lemma}[Classical weighted quantile characterization]
\label{lem:weighted_pinball_quantile}
Under the preceding finite-score and positive-weight conditions, a threshold
$q$ minimizes $L_{\boldsymbol\beta}$ if and only if
\begin{equation}
 \widehat F_{\boldsymbol\beta}(q^-)
 \le\tau\le\widehat F_{\boldsymbol\beta}(q).
 \label{eq:weighted_quantile_condition}
\end{equation}
The lower weighted empirical quantile in \eqref{eq:practical_quantile} is
the smallest minimizer,
\begin{equation}
 \widehat q_{\boldsymbol\beta}
 =\min\operatorname*{arg\,min}_{q\in\R}L_{\boldsymbol\beta}(q).
 \label{eq:jts_quantile_regression_lower}
\end{equation}
\end{lemma}

Appendix~\ref{app:optimization} gives the subgradient argument.
If cumulative weight equals $\tau$ at a score, the loss can be constant up
to the next distinct score. Taking the smallest minimizer selects the
lower quantile consistently.

JTS-SCB applies an outer maximization to these selected solutions,
\begin{equation}
 \boxed{\displaystyle
 \widehat\theta_{\rm JTS}
 =\sup_{\boldsymbol\beta\in\mathcal B}
    \left[\min\operatorname*{arg\,min}_{q\in\R}
       L_{\boldsymbol\beta}(q)\right]
 =\sup_{\boldsymbol\beta\in\mathcal B}\widehat q_{\boldsymbol\beta}.}
 \label{eq:jts_threshold}
\end{equation}
The same $\boldsymbol\beta$ determines the scores and weights inside each
inner problem. This is the joint score--weight coupling in the
optimization formulation. Multiplying all weights at a fixed tilt by a
positive constant leaves its quantile solution unchanged.

The equivalent constrained min--max formulation chooses the smallest
threshold that is adequate for every candidate tilt,
\begin{equation}
 \boxed{\displaystyle
 \widehat\theta_{\rm JTS}
 =\inf\left\{q\in\R\ \middle|\
       \sup_{\boldsymbol\beta\in\mathcal B}
       \bigl[\tau-\widehat F_{\boldsymbol\beta}(q)\bigr]\le0
       \right\}.}
 \label{eq:jts_minmax_feasibility}
\end{equation}
At a proposed threshold, the inner supremum finds the largest weighted
empirical calibration shortfall. The outer minimization raises the
threshold only until that shortfall is nonpositive. Equivalence follows
from $\widehat q_{\boldsymbol\beta}\le q$ if and only if
$\widehat F_{\boldsymbol\beta}(q)\ge\tau$, as detailed in
Appendix~\ref{app:optimization}. If
$\sup_{\boldsymbol\beta}\widehat q_{\boldsymbol\beta}<\infty$, the infimum
in \eqref{eq:jts_minmax_feasibility} is attained and is the minimum common
feasible threshold. If the supremum is infinite, no finite threshold is
feasible and both expressions equal $+\infty$.

Equations~\eqref{eq:jts_threshold} and \eqref{eq:jts_minmax_feasibility}
express the same calibration rule as nested quantile optimization and
robust feasibility. The outer step in \eqref{eq:jts_threshold} maximizes
the selected quantile, not the minimized pinball-loss value. Minimizing
$\sup_{\boldsymbol\beta}L_{\boldsymbol\beta}(q)$ would generally define a
different procedure. We make no interchange-of-extrema or saddle-point
claim for $L_{\boldsymbol\beta}$. The constraint is empirical and uses
only calibration weights, so it does not itself establish target coverage.

\begin{figure}[t]
\centering
\begin{tikzpicture}[
 font=\small,
 box/.style={draw,rounded corners=3pt,align=center,inner sep=7pt,
             text width=4.05cm,minimum height=1.35cm},
 arr/.style={-{Latex[length=2.2mm]},thick}]
\node[box,fill=blue!6] (family) at (0,0)
 {Fitted source predictive\\and plausible tilts $\mathcal B$};
\node[box,fill=orange!9] (pairs) at (5.0,0)
 {One tilt $\boldsymbol\beta$\\generates both $S_i(\boldsymbol\beta)$\\and $W_i(\boldsymbol\beta)$};
\node[box,fill=orange!9] (qr) at (10.0,0)
 {Weighted quantile\\regression\\$\widehat q_{\boldsymbol\beta}
 =\min\arg\min_q L_{\boldsymbol\beta}(q)$};
\node[box,fill=purple!7] (common) at (10.0,-2.4)
 {Common threshold\\$\widehat\theta_{\rm JTS}
 =\sup_{\boldsymbol\beta}\widehat q_{\boldsymbol\beta}$};
\node[box,fill=purple!7] (prediction) at (5.0,-2.4)
 {Union over score geometries\\at the common threshold\\JTS-SCB};
\node[box,fill=green!7] (union) at (0,-2.4)
 {Keep each threshold\\Union over tilts\\JTS-Union};
\draw[arr] (family)--(pairs);
\draw[arr] (pairs)--(qr);
\draw[arr] (qr)--(common);
\draw[arr] (common)--(prediction);
\draw[arr] (qr.south) -- ++(0,-.35) -| (union.north);
\end{tikzpicture}
\caption{The joint calibration problem and its two prediction constructions.
Each inner quantile-regression problem uses a coherent score--weight pair.
JTS-SCB maximizes the selected quantile before unioning score sublevel sets.
JTS-Union carries the individual quantiles into prediction.}
\label{fig:jts_schematic}
\end{figure}

\subsection{Prediction sets and JTS-Union}
\label{sec:practical_sets}

After solving the joint calibration problem, JTS-SCB forms the practical
prediction set
\begin{equation}
  \CC_{\rm JTS}^{\rm scb}(x)
  =\bigcup_{\boldsymbol\beta\in\mathcal B}
   \left\{y:S_{\boldsymbol\beta}(x,y)
                \le\widehat\theta_{\rm JTS}\right\}.
  \label{eq:jts_scb}
\end{equation}
JTS-SCB has a closed-form interval in the scalar Gaussian model. The tilt
requiring the largest calibration threshold need not produce a prediction
region containing those of the other tilts, because the score geometry
also changes. Calibration selects one common threshold, while prediction
still unions the sublevel sets of all candidate scores.

For example, in the linear Gaussian model the fixed-threshold interval
is centered at $\mub(x)+\beta\sigb^2(x)$. A tilt with a smaller calibration
quantile can still move this interval farther to the left or right than
the tilt with the largest quantile. Consequently,
\[
 \widehat q_{\beta_1}\le\widehat q_{\beta_2}
 \quad\not\Longrightarrow\quad
 \{y:S_{\beta_1}(x,y)\le\widehat q_{\beta_1}\}
 \subseteq\{y:S_{\beta_2}(x,y)\le\widehat q_{\beta_2}\}.
\]
A maximizing tilt, when one exists, is neither an estimate of the target
nor a sufficient summary of the family for prediction.

The beta-specific union
\begin{equation}
  \CC_{\rm JTS}^{\rm union}(x)
  =\bigcup_{\boldsymbol\beta\in\mathcal B}
   \left\{y:S_{\boldsymbol\beta}(x,y)
                \le\widehat q_{\boldsymbol\beta}\right\}
  \label{eq:beta_specific_union}
\end{equation}
instead retains the threshold associated with each score geometry. We call
this refinement \emph{JTS-Union}.

Because
$\widehat q_{\boldsymbol\beta}\le\widehat\theta_{\rm JTS}$ for every tilt,
\begin{equation}
  \CC_{\rm JTS}^{\rm union}(x)
  \subseteq \CC_{\rm JTS}^{\rm scb}(x).
  \label{eq:envelope_inclusion}
\end{equation}
Thus JTS-SCB is a common-threshold envelope of JTS-Union. Neither practical set
includes the candidate response's own importance weight, so neither inherits
the exact weighted-conformal guarantee of Section~\ref{sec:validity}.

\begin{proposition}[Monotonicity under nested uncertainty sets]
\label{prop:monotonicity}
Use the same score--weight family, calibration sample, and level for both
nonempty budgets. If $\mathcal B_1\subseteq\mathcal B_2$, then, for every calibration
sample and test input $x$,
\begin{align*}
  \CC_{{\rm JTS},1}^{\rm union}(x)
  &\subseteq\CC_{{\rm JTS},2}^{\rm union}(x),\\
  \CC_{{\rm JTS},1}^{\rm scb}(x)
  &\subseteq\CC_{{\rm JTS},2}^{\rm scb}(x).
\end{align*}
Consequently, under any fixed target distribution, the marginal coverage probabilities
of both practical sets are nondecreasing as the sensitivity set expands.
\end{proposition}

This elementary nesting follows by enlarging the union and the supremum,
as shown in Appendix~\ref{app:optimization}. Thus the budget is a monotone
sensitivity parameter even when the target lies outside the family.
Monotonicity alone does not identify a budget that achieves nominal coverage.

\subsection{Gaussian geometry and the cost of sensitivity}
\label{sec:geometry}

The fixed-tilt Gaussian formula of \citet{Choi2026eiml} gives an explicit
common-threshold envelope by taking a union of intervals with the same
half-width and linearly moving centers. The following direct calculation
isolates the center-spread contribution to its width.

\begin{proposition}[Closed-form Gaussian sensitivity envelope]
\label{prop:gaussian_interval}
Assume the scalar linear Gaussian score \eqref{eq:gaussian_score}, with
$0<\sigb^2(x)<\infty$, and let $\mathcal B=[\beta_{\rm lo},\beta_{\rm hi}]$
have finite endpoints. Assume the calibration scores are finite and
continuous in $\beta$, so $\widehat\theta_{\rm JTS}<\infty$. Define
\begin{equation*}
  r(x)=\widehat\theta_{\rm JTS}
       -\frac12\log\{2\pi\sigb^2(x)\}.
\end{equation*}
If $r(x)<0$, then $\CC_{\rm JTS}^{\rm scb}(x)=\varnothing$. If $r(x)\ge0$,
let $h(x)=\{2\sigb^2(x)r(x)\}^{1/2}$. Then
\begin{equation}
  \CC_{\rm JTS}^{\rm scb}(x)
  =\left[
    \mub(x)+\beta_{\rm lo}\sigb^2(x)-h(x),
    \ \mub(x)+\beta_{\rm hi}\sigb^2(x)+h(x)
    \right].
  \label{eq:closed_form_interval}
\end{equation}
Its width is
\begin{equation}
  2h(x)+(\beta_{\rm hi}-\beta_{\rm lo})\sigb^2(x).
  \label{eq:closed_form_width}
\end{equation}
\end{proposition}

The union of the fixed-tilt Gaussian intervals has a common radius $h(x)$ and a center
spread $(\beta_{\rm hi}-\beta_{\rm lo})\sigb^2(x)$. The radius uses the
largest calibration threshold over the budget, so $2h(x)$ is not generally
the width of the ordinarily calibrated fixed-tilt set. The second term is
therefore only the center-spread component of the sensitivity cost.
JTS-Union is contained in JTS-SCB but generally requires numerical unioning
of intervals with different half-widths.

The common-threshold construction uses the specified negative log density
scale. Unlike JTS-Union, it is not invariant to a separate strictly increasing
transformation of each tilt's score. Such transformations preserve
within-tilt rankings but change comparisons across tilt-specific thresholds.

The two sources of extra width have different interpretations. Enlarging
$\mathcal B$ may raise the common calibration threshold, increasing the
radius $h(x)$. It also spreads the candidate centers. For a symmetric
budget $[-B,B]$, the center-spread term alone is $2B\sigb^2(x)$.
For a one-sided budget $[0,B]$, it is $B\sigb^2(x)$ and the interval moves
asymmetrically. The decomposition therefore makes the sensitivity cost
input dependent even when the same budget is used everywhere.

For a quadratic tilt and any threshold $q$, define
\[
 r_{\boldsymbol\beta}(x;q)
 =q-\tfrac12\log\{2\pi\sigb^2_{\boldsymbol\beta}(x)\}.
\]
The fixed-tilt set is empty when $r_{\boldsymbol\beta}(x;q)<0$ and is
otherwise the interval with center $\mub_{\boldsymbol\beta}(x)$ and
half-width
$\{2\sigb^2_{\boldsymbol\beta}(x)r_{\boldsymbol\beta}(x;q)\}^{1/2}$.
Both center and radius now vary with the tilt. Numerical interval unioning
is sufficient for either practical construction. An interval hull may be
reported when a single interval is desired, but it can fill gaps in the
union and increase coverage and width.

\subsection{Coupling the score and the weight}
\label{sec:coupling}

The uncertainty class is the set of paired vectors
\[
 \bigl\{(S_1(\boldsymbol\beta),\ldots,S_n(\boldsymbol\beta),
         W_1(\boldsymbol\beta),\ldots,W_n(\boldsymbol\beta))
         \mid\boldsymbol\beta\in\mathcal B\bigr\}.
\]
Taking a worst score and a worst weight separately enlarges this class in
an uncontrolled way. A simple Gaussian example makes the issue concrete.
Let the fitted mean be zero, predictive variance one, source label marginal
$\N(0,1)$, calibration response $y=1$, and budget $[-1,1]$. Then
\[
 S_\beta(x,1)=\tfrac12(1-\beta)^2+\tfrac12\log(2\pi),
 \qquad W_\beta(1)=e^{\beta-\beta^2/2}.
\]
The score is maximized at $\beta=-1$, whereas the weight is maximized at
$\beta=1$. No member of the family supplies both maxima. Taking such
pointwise maxima across all observations can also use a different tilt
at each observation.

JTS-SCB keeps a single parameter throughout each calibration problem and
aggregates only its solution. This is a modeling coherence requirement,
not a necessary condition for every valid weighted conformal procedure.
One can calibrate a separately chosen fixed score with correct weights.
What fails for independently maximized ingredients is their interpretation
as members of the specified tilted predictive family, and their
identification with the JTS optimization.

A useful ablation holds the score fixed at $S_0$ and varies only the weights.
For a finite nonempty computational family $\mathcal G$, let
$\widehat q^{\rm IW}_{\boldsymbol\beta}$ be the lower weighted empirical
quantile of $S_0(X_i,Y_i)$ using $W_i(\boldsymbol\beta)$. Define
\begin{equation}
 \CC_{\rm RIW}(x)
 =\left\{y:S_0(x,y)\le
      \max_{\boldsymbol\beta\in\mathcal G}
                 \widehat q^{\rm IW}_{\boldsymbol\beta}\right\}
 =\bigcup_{\boldsymbol\beta\in\mathcal G}
      \{y:S_0(x,y)\le\widehat q^{\rm IW}_{\boldsymbol\beta}\}.
 \label{eq:robust_iw}
\end{equation}
We call this baseline Robust-IW. The equality holds because all sublevel
sets use the same score and the largest quantile is attained on a finite
family. Its common-threshold and union versions therefore coincide.
For a continuum, attainment is needed for this exact closed-set identity
at the supremum boundary. Robust-IW is a coherent fixed-score procedure
with weight sensitivity, and it omits candidate weights just as the
practical JTS constructions do. There is no general width ordering between
it and JTS. Their paired empirical comparison tests the effect of varying
the predictive geometry.

\subsection{Computation}
\label{sec:computation}

Algorithm~\ref{alg:jts} implements the joint problem on a finite grid.
Sorting is an exact solver for weighted intercept-only quantile regression.
A generic quantile-regression solver is therefore unnecessary for the inner
step, although its linear-programming form in
Appendix~\ref{app:optimization} makes the optimization explicit.

\begin{algorithm}[t]
\caption{Joint tilt-sensitivity calibration on a finite family}
\label{alg:jts}
\textbf{Input}\quad Fitted predictive density, calibration sample, level
$\alpha$, and finite nonempty family $\mathcal G\subset\mathcal B$.
\begin{enumerate}[leftmargin=*,itemsep=2pt,topsep=4pt]
\item For each $\boldsymbol\beta\in\mathcal G$, form
$S_i(\boldsymbol\beta)$ and $W_i(\boldsymbol\beta)$ using the same tilt.
Normalize the weights and sort the paired scores and weights by score.
\item At each tilt, set $\widehat q_{\boldsymbol\beta}$ to the first score at which
cumulative weight reaches $1-\alpha$.
\item Set $\widehat\theta_{\mathcal G}
=\max_{\boldsymbol\beta\in\mathcal G}\widehat q_{\boldsymbol\beta}$.
\item For a test input $x$, union
$\{y:S_{\boldsymbol\beta}(x,y)\le\widehat\theta_{\mathcal G}\}$
for JTS-SCB. Alternatively union
$\{y:S_{\boldsymbol\beta}(x,y)\le\widehat q_{\boldsymbol\beta}\}$
for JTS-Union.
\end{enumerate}
\textbf{Output}\quad The union, or its interval hull when explicitly requested.

\end{algorithm}

With $K$ candidate tilts, calibration costs $O(Kn\log n)$ apart from
predictive evaluation. Log weights can be shifted by a common constant
before exponentiation to stabilize normalization. The weighted quantile is
nonsmooth where score orderings change or accumulated weight crosses the
quantile level. Neither smoothness nor unimodality of the outer objective
is assumed. The scalar experiments use global grids with 61 points for the
synthetic study and 41 for real data. The quadratic study uses a
$7\times7$ grid.

The feasibility formulation in \eqref{eq:jts_minmax_feasibility} also permits
bisection on the threshold within a valid bracket, provided each check
certifies whether $\inf_{\boldsymbol\beta\in\mathcal B}
\widehat F_{\boldsymbol\beta}(q)\ge1-\alpha$.
For scalar linear Gaussian tilts at fixed $q$, each condition
$S_i(\beta)\le q$ defines an interval or the empty set in $\beta$.
Their endpoints partition the budget into regions with fixed included
calibration indices, within which $\widehat F_\beta(q)$ is a smooth ratio
of exponential sums. This suggests checking interior minima, breakpoints,
and their one-sided limits, but requires a globally certified inner search.
The reported experiments use the grid procedure rather than this alternative.

A finite grid is a distinct computational family. Its common threshold
cannot exceed the continuum supremum. At a grid threshold, the scalar
Gaussian endpoint formula gives the \emph{interval hull} of the grid
sublevel sets when the grid contains both budget endpoints. It need not
be the grid union itself if intervals are disjoint. The real-data experiments
merge the intervals for JTS-Union and report their actual total length.
JTS-SCB reports the common-threshold interval hull. The synthetic JTS-SCB
procedures likewise report hulls. Taking a hull can only enlarge a set,
but does not certify approximation to the full continuum method.

Any exact or corrected coverage statement applies to a grid when the true
tilt is on that grid, or when a certified discretization bound has been
supplied. The true tilts in the primary comparisons are grid points. A
continuous deployment parameter need not lie on the grid. The separate
budget-misspecification comparison deliberately excludes the truth. These
limitations are separate from the omission of candidate weights in
practical calibration.

\section{Validity and limitations of the calibration rule}
\label{sec:validity}

The joint optimization enforces weighted empirical calibration over the
family. This section distinguishes that requirement from target coverage.
The exact weighted conformal benchmark includes a candidate weight that
is absent from the practical quantile problems. Its effect explains both a
limitation of the practical procedure and the cost of a sufficient correction.

\subsection{The exact candidate-weighted benchmark}
\label{sec:exact}

For a candidate $y$ at input $x$, write
$S_{n+1}(y;\boldsymbol\beta)=S_{\boldsymbol\beta}(x,y)$,
$W_{n+1}(y;\boldsymbol\beta)=w(y;\boldsymbol\beta)$, and
$C_{\boldsymbol\beta}=\sum_{i=1}^nW_i(\boldsymbol\beta)$.
The exact weighted conformal $p$-value and set are
\begin{align}
 \pi_{\boldsymbol\beta}(y)
 &=\frac{\sum_{i=1}^n W_i(\boldsymbol\beta)
       \1\{S_i(\boldsymbol\beta)\ge S_{n+1}(y;\boldsymbol\beta)\}
          +W_{n+1}(y;\boldsymbol\beta)}
        {C_{\boldsymbol\beta}+W_{n+1}(y;\boldsymbol\beta)},
 \label{eq:exact_pvalue}\\
 \CC^{\rm exact}_{\boldsymbol\beta}(x)
 &=\{y:\pi_{\boldsymbol\beta}(y)>\alpha\},\qquad
 \CC^{\rm exact}_{\mathcal B}(x)
 =\bigcup_{\boldsymbol\beta\in\mathcal B}
        \CC^{\rm exact}_{\boldsymbol\beta}(x).
 \label{eq:exact_union}
\end{align}
The candidate term is required under both covariate and label shift. Here
it varies with the unobserved response and must remain inside the inversion.
The word \emph{exact} refers to a finite-sample coverage lower bound, not
equality of coverage with $1-\alpha$.

\begin{assumption}[Sampling and a fixed tilt family]
\label{ass:validity}
Let $\HH$ contain the training data and any reference information used to
choose the nonempty family $\mathcal B$ and its scores, but no calibration
or test draws. Conditional on $\HH$, the family is fixed and measurable,
the calibration observations are i.i.d.\ from $P_s$, and the test point is
independent of them with distribution $P_t\ll P_s$. These distributions
may depend on $\HH$, as when an empirical reservoir is fixed first. For
an $\HH$-measurable $\boldsymbol\beta^\star\in\mathcal B$, almost surely,
\[
 \frac{dP_t}{dP_s}(x,y)=w(y;\boldsymbol\beta^\star).
\]
Weights are finite and positive and scores are finite on the evaluated
domain. All displayed unions, suprema, and events are measurable.
\end{assumption}

\begin{corollary}[Coverage inherited from weighted exchangeability]
\label{cor:exact_coverage}
Under Assumption~\ref{ass:validity},
\begin{equation}
 \Prob\{Y_{\rm test}\in
 \CC^{\rm exact}_{\mathcal B}(X_{\rm test})\mid\HH\}\ge1-\alpha.
 \label{eq:exact_coverage}
\end{equation}
\end{corollary}

Apply \citet[Theorem~2]{TibshiraniR2019neurips} to the full observation
$Z=(X,Y)$ with weight functions $w_1=\cdots=w_n=1$ and
$w_{n+1}(x,y)=w(y;\boldsymbol\beta^\star)$, and then use containment of
the true-tilt set in the union. The cited theorem already allows weights
on the full observation, so response-dependent weighting needs no extension
of its coverage argument. No simultaneous testing argument is needed.
Appendix~\ref{app:validity} gives the weighted-rank proof with
ties. The guarantee is conditional on $\HH$ and marginal over calibration
and test draws. It is not conditional on the realized calibration sample
or test input, and it does not require a correct predictive model.

\subsection{What omitting the candidate weight changes}
\label{sec:omission}

Both the practical and exact sets can be expressed through one statistic,
\begin{equation}
 A_{\boldsymbol\beta}(y)
 =\sum_{i=1}^nW_i(\boldsymbol\beta)
       \1\{S_i(\boldsymbol\beta)\ge S_{n+1}(y;\boldsymbol\beta)\}.
 \label{eq:upper_score_mass}
\end{equation}
For every $a\in(0,1)$, the lower weighted quantile and the exact $p$-value
satisfy
\begin{align}
 \{y:S_{\boldsymbol\beta}(x,y)\le\widehat q_{\boldsymbol\beta}(a)\}
 &=\{y:A_{\boldsymbol\beta}(y)>aC_{\boldsymbol\beta}\},
 \nonumber\\
 \CC^{\rm exact}_{\boldsymbol\beta}(x)
 &=\{y:A_{\boldsymbol\beta}(y)>
       \alpha C_{\boldsymbol\beta}-(1-\alpha)W_{n+1}(y;\boldsymbol\beta)\}.
 \label{eq:one_variable}
\end{align}
For the first identity, writing $s=S_{\boldsymbol\beta}(x,y)$ gives
$A_{\boldsymbol\beta}(y)/C_{\boldsymbol\beta}
=1-\widehat F_{\boldsymbol\beta}(s^-)$, and a finite empirical distribution
satisfies $s\le\widehat q_{\boldsymbol\beta}(a)$ exactly when
$\widehat F_{\boldsymbol\beta}(s^-)<1-a$.
The second identity rearranges $\pi_{\boldsymbol\beta}(y)>\alpha$.
These elementary inversions retain score ties and cumulative-weight
boundary cases through their strict inequalities.

It follows that every fixed-tilt practical set is a subset of its exact
counterpart and
$\CC^{\rm union}_{\rm JTS}(x)\subseteq\CC^{\rm exact}_{\mathcal B}(x)$.
There is no general inclusion in either direction between the exact union
and the common-threshold JTS-SCB set. The common threshold can offset some
of the omission empirically, but that does not give a coverage theorem.
Even the singleton no-shift practical set lacks the usual
$(n+1)$-based conformal rank correction.

On the same calibration and test draws, exact fixed-tilt inference accepts
while calibration-only inference rejects precisely when
\begin{equation}
 \alpha C_{\boldsymbol\beta}-(1-\alpha)W_{n+1}(y;\boldsymbol\beta)
 <A_{\boldsymbol\beta}(y)\le\alpha C_{\boldsymbol\beta}.
 \label{eq:omitted_weight_diagnostic}
\end{equation}
Its frequency gives the paired coverage difference without constructing
the full exact set. This difference is not generally the shortfall below
nominal coverage, because the exact procedure can be conservative.
Section~\ref{sec:real_omission} measures this event at the true tilt.

\subsection{Candidate weights and response tails}
\label{sec:tails}

Define $K_{\boldsymbol\beta}=\alpha C_{\boldsymbol\beta}/(1-\alpha)$.
For every candidate, nonnegativity of the calibration contribution gives
\begin{equation}
 \pi_{\boldsymbol\beta}(y)\ge
 \frac{W_{n+1}(y;\boldsymbol\beta)}
 {C_{\boldsymbol\beta}+W_{n+1}(y;\boldsymbol\beta)},\qquad
 \{y:W_{n+1}(y;\boldsymbol\beta)>K_{\boldsymbol\beta}\}
 \subseteq\CC^{\rm exact}_{\boldsymbol\beta}(x).
 \label{eq:weight_forces_acceptance}
\end{equation}
If the candidate score exceeds every calibration score, all calibration
indicators vanish and the inclusion becomes the equivalence
\begin{equation}
 y\in\CC^{\rm exact}_{\boldsymbol\beta}(x)
 \quad\Longleftrightarrow\quad
 W_{n+1}(y;\boldsymbol\beta)>K_{\boldsymbol\beta}.
 \label{eq:tail_acceptance}
\end{equation}
The condition $W_{n+1}>K_{\boldsymbol\beta}$ is the infinite-quantile
condition in the augmented weighted quantile of
\citet[Eq.~(7)]{TibshiraniR2019neurips}. A covariate-dependent candidate
weight is constant over responses at fixed $x$. Here the same condition
varies along the response line. The following elementary consequences
classify that geometry for the specified exponential families.

\begin{proposition}[Tail geometry of exact inversion]
\label{prop:tails}
Fix $n\ge1$, $\alpha\in(0,1)$, a finite calibration sample, and a test input
$x$. Assume finite calibration scores and finite positive weights, and
invert over $\R$ using the specified exponential weight formulas.
In parts (ii) and (iii), use the Gaussian scores from
\eqref{eq:quadratic_predictive}, with normalizable ratios and positive finite
source predictive variances at every calibration input and at $x$.
\begin{enumerate}[label=(\roman*),leftmargin=*,itemsep=2pt]
\item For $w(y;\beta)=c_\beta e^{\beta y}$ with
$0<c_\beta<\infty$ and $\beta\ne0$, the exact set contains an unbounded
half-line for any finite-valued score. Any union containing that candidate
is therefore unbounded, even if the true shift is zero.
\item For each fixed admissible quadratic tilt, $\beta_2<0$ gives a
bounded fixed-tilt exact set, whereas $\beta_2>0$
gives acceptance of both sufficiently extreme tails. At $\beta_2=0$,
$\beta_1\ne0$ gives the linear case. At $(0,0)$ the set is bounded if
$\alpha\ge1/(n+1)$ and equals $\R$ otherwise.
\item If $\mathcal B_Q$ is nonempty, compact, and contained in
$\{\beta_2<0\}$, then $\CC^{\rm exact}_{\mathcal B_Q}(x)$ is bounded.
If $\boldsymbol\beta^\star\in\mathcal B_Q$ and
Assumption~\ref{ass:validity} holds with $\mathcal B=\mathcal B_Q$,
this bounded union also has coverage at least $1-\alpha$.
\end{enumerate}
\end{proposition}

For a growing weight, \eqref{eq:weight_forces_acceptance} forces acceptance
without a score-growth assumption. A vanishing weight gives eventual
rejection when the score diverges. Compactness in the negative-quadratic
case supplies a uniform tail cutoff, which is needed because a union of
individually bounded sets need not be bounded. The detailed argument is in
Appendix~\ref{app:validity}.

With unnormalized linear weights, the crossing location is particularly
simple,
\begin{equation}
 y_W(\beta)=\frac1\beta\log\left\{
 \frac{\alpha}{1-\alpha}\sum_{i=1}^ne^{\beta Y_i}\right\}.
 \label{eq:linear_crossing}
\end{equation}
Every $y>y_W(\beta)$ is accepted when $\beta>0$, and every
$y<y_W(\beta)$ is accepted when $\beta<0$. Thus even a known nonzero
linear tilt has this behavior. For a small nonzero tilt, the forced-acceptance
tail may lie far outside a plotted response range. Apparent boundedness on a
finite response grid therefore does not establish boundedness of the exact set.

The restricted negative-quadratic regime
requires prior knowledge that plausible tilts suppress both tails.
For Gaussian marginals it corresponds to a target variance smaller than
the source variance.

These statements describe the specified inversion over $\R$. Restriction
to known bounded support preserves coverage and removes unboundedness.
For a finite source reservoir the ratio is identified on its support,
and using the exponential formula outside that support is an extension.
The real-data experiments therefore assess candidate-weight omission,
not the continuum tail conclusion itself.

\subsection{A sufficient correction and its cost}
\label{sec:correction}

A known upper bound $M_{\boldsymbol\beta}(D)$ on candidate weights over a
prespecified response range $D$ gives a conservative correction of the
same inner quantile problems. It replaces $\alpha$ by
$\alpha'_{\boldsymbol\beta}(D)
=\alpha-(1-\alpha)M_{\boldsymbol\beta}(D)/C_{\boldsymbol\beta}$.
When every corrected level is positive, we use the corrected quantiles
in JTS-Union or take their supremum for JTS-SCB. Otherwise the corresponding
fallback is $D$ union the practical set. Outputs are not truncated to $D$.
This is a localized application of the bounded-candidate-weight quantile
of \citet{JinY2023pnas}, rather than a new calibration principle.

Proposition~\ref{prop:correction} in Appendix~\ref{app:correction} gives the
construction and its proof. Each corrected set contains its practical
counterpart everywhere and the exact union on $D$. If the family contains
the true tilt and target probability outside $D$ is at most $\bar\delta_D$,
the resulting coverage lower bound is $1-\alpha-\bar\delta_D$, including
fallback samples. These requirements can be costly. The compact evaluation
in Table~\ref{tab:correction_cost} shows that a finite weight bound does not
ensure feasibility, moderate width, or a useful family-wide tail bound.

\FloatBarrier
\section{Experiments}
\label{sec:experiments}

Controlled label shifts on three real regression datasets provide the
principal evidence. We compare coverage and width, isolate score--weight
coupling with an untilted-score baseline, and test a family that excludes
the true tilt. A paired diagnostic measures the effect of candidate-weight
omission. Two synthetic studies then isolate predictive bias and quadratic
tilt estimation with limited target inputs.

All comparisons use nominal coverage $.90$ and share calibration and test
draws within each trial. Coverage is the fraction of target responses
included, and width is mean Lebesgue length. Oracle-WT knows the true tilt
but uses the same calibration-only quantile as the practical methods. It
is an efficiency benchmark, not an exactly valid procedure. The
candidate-weighted exact rule appears only in the omission diagnostic.
Monte Carlo standard errors are trial-level standard deviations divided
by the square root of the trial count, with trialwise differences used
for paired comparisons. Real-data errors condition on the observed dataset.

\subsection{Controlled label shifts on real regression data}
\label{sec:real}

We use AqSolDB aqueous solubility \citep{SorkunM2019sd}, California housing,
and superconducting critical temperature \citep{HamidiehK2018cms}.
Each of 60 trials selects 1500 training observations and treats the
remaining rows as a fixed empirical source reservoir. Responses are log
solubility, log house value, and $\log(1+T_c)$, standardized using training
statistics. We sample 1000 calibration observations uniformly with
replacement from the reservoir and, independently, 400 target observations
with probabilities proportional to $\exp(\beta^\star y)$. This gives
exact label shift relative to the reservoir. All target tilts share the
calibration draw within a trial.

The Gaussian ridge predictor uses standardized inputs and 100 random
Fourier features. Mean held-out $R^2$ is $.748$, $.687$, and $.773$ for
the three datasets. Appendix~\ref{app:real_details} gives preprocessing
and fitting details. The design controls the shift rather than asserting
that a natural deployment shift follows this model. Its directional
budgets are supplied experimental conditions, not learned uncertainty sets.

\paragraph{Methods and information supplied.}
Source-CB uses the untilted score and unit calibration weights. Oracle-WT
uses the true tilt in both score and weights. TPS fits one tilt from
300 independent unlabeled target inputs using analytic predictive
expectations and damped moment updates. JTS-SCB, JTS-Union, and Robust-IW
use 41-point grids on $[-2,0]$ for negative shifts and $[0,2]$ for positive
shifts. Robust-IW varies only the weights, keeping the untilted score as
in \eqref{eq:robust_iw}. TPS receives the same directional bounds.
All primary true tilts are grid points.

All these methods use lower empirical $.90$ quantiles, so Source-CB also
omits the usual finite-sample rank adjustment. JTS-Union merges the
tilt-specific intervals and reports their total length. JTS-SCB reports
the interval hull at the largest grid threshold. TPS uses its final
iterate, including boundary and nonconverged fits. No trial is discarded.

\paragraph{Coverage and width.}
Table~\ref{tab:real_common_main} gives the main comparison. At
$\beta^\star=-1.5$, Source-CB covers $.735$, $.716$, and $.754$ on the
three datasets. Oracle-WT improves these to $.892$, $.871$, and $.893$.
Because the oracle already knows the tilt, its remaining shortfall cannot
be attributed to shift estimation.

TPS is within one percentage point of oracle coverage in six of nine
conditions. Its largest shortfall occurs at the positive superconductivity
shift, with coverage $.868$ against $.905$. The mean fitted tilt is
$1.576$ rather than $1$, and $42\%$ of fits reach the upper bound.
Plug-in efficiency thus depends on the reliability of tilt estimation
under the fitted predictor.

JTS-SCB covers $.916$--$.978$, with larger widths. JTS-Union is
approximately $2$--$12\%$ narrower and covers $.909$--$.976$. On AqSolDB
at $-1.5$, retaining individual thresholds reduces width from $2.86$ to
$2.53$ and coverage from $.940$ to $.922$. The sensitivity premium remains
substantial. On California housing at $-1.0$, Oracle-WT has width $2.12$,
against $4.04$ for JTS-Union and $4.49$ for JTS-SCB. Both sensitivity
procedures must also accommodate the more extreme candidates down to $-2$.
Higher coverage should therefore be assessed together with the cost of
the supplied family.

\begin{table}[t]
\centering
\caption{Controlled real-data shifts, 60 paired trials. Entries are
coverage/mean width in standardized response units. Robust-IW varies only
weights, while JTS varies scores and weights. Robust-IW's union and
common-threshold versions coincide. Coverage Monte Carlo standard errors are at most
$.010$, conditional on each observed dataset.}
\label{tab:real_common_main}
\footnotesize
\setlength{\tabcolsep}{3pt}
\begin{tabular}{llrrrrrr}
\toprule
Data & $\beta^\star$ & Source-CB & Oracle-WT & TPS & Robust-IW & JTS-SCB & JTS-Union \\
\midrule
AqSolDB & $-1.5$ & .735/1.58 & .892/2.07 & .889/2.06 & .905/2.89 & .940/2.86 & .922/2.53 \\
AqSolDB & $-1.0$ & .814/1.58 & .900/1.88 & .897/1.87 & .940/2.87 & .962/2.84 & .945/2.51 \\
AqSolDB & $+1.0$ & .905/1.58 & .900/1.43 & .891/1.45 & .918/1.69 & .970/2.10 & .968/2.05 \\
California & $-1.5$ & .716/1.77 & .871/3.36 & .857/2.92 & .898/4.70 & .916/4.53 & .909/4.08 \\
California & $-1.0$ & .821/1.76 & .899/2.12 & .892/2.03 & .957/4.67 & .969/4.49 & .962/4.04 \\
California & $+1.0$ & .883/1.76 & .896/1.67 & .880/1.72 & .927/2.13 & .973/2.51 & .969/2.44 \\
Supercond. & $-1.5$ & .754/1.52 & .893/1.87 & .890/1.86 & .913/2.40 & .951/2.44 & .938/2.21 \\
Supercond. & $-1.0$ & .802/1.51 & .899/1.81 & .897/1.80 & .934/2.39 & .961/2.42 & .945/2.19 \\
Supercond. & $+1.0$ & .958/1.47 & .905/1.12 & .868/1.13 & .958/1.47 & .978/1.91 & .976/1.75 \\
\bottomrule
\end{tabular}
\end{table}

\paragraph{What coupling changes.}
For all six negative-shift conditions, JTS-Union is about $8$--$13\%$
narrower than Robust-IW and has higher observed coverage. At $-1.5$, the
paired width reductions are $.361$ (SE $.044$), $.628$ (SE $.066$), and
$.196$ (SE $.026$). JTS-SCB is narrower than Robust-IW on AqSolDB and
California housing but slightly wider on superconductivity, indicating
that its common threshold can consume some of the gain from changing
score geometry.

For positive shifts, both coupled procedures are wider than Robust-IW.
JTS-Union adds about $14$--$21\%$ width and increases coverage, although
Robust-IW already covers $.918$, $.927$, and $.958$. At the positive
superconductivity shift, $\beta=0$ uniquely maximizes the untilted-score
weighted quantile in all 60 trials, explaining the equality of Robust-IW
and Source-CB. Coupling improves the observed width--coverage comparison
for negative shifts here, but has no uniform efficiency advantage.

\subsection{A budget that excludes the true tilt}
\label{sec:wrong_budget}

At $\beta^\star=-1.5$, we restrict the primary grid to its 21-point subset
$[-1,0]$, keeping the predictor, calibration observations, and target draws
fixed. Table~\ref{tab:wrong_budget} shows the resulting loss of protection.
JTS-Union coverage falls from $.922$ to $.872$ on AqSolDB, from $.909$ to
$.836$ on California housing, and from $.938$ to $.892$ on superconductivity.
JTS-SCB also falls below $.90$ on the first two datasets, although it
remains near nominal on superconductivity. Excluding the truth can reduce
both width and coverage without forcing undercoverage in every condition.
These results do not supply a rule for choosing the family from unlabeled
target inputs.

\begin{table}[t]
\centering
\caption{Misspecified budget $[-1,0]$ at $\beta^\star=-1.5$, with the same
60 trials as the corresponding rows of Table~\ref{tab:real_common_main}.
Entries are coverage/mean width. The 21-point grid is a subset of the
primary grid.}
\label{tab:wrong_budget}
\small
\begin{tabular}{lrrr}
\toprule
Data & Robust-IW & JTS-SCB & JTS-Union \\
\midrule
AqSolDB & .843/2.14 & .879/2.16 & .872/2.00 \\
California & .807/2.40 & .841/2.47 & .836/2.30 \\
Supercond. & .868/2.05 & .901/2.07 & .892/1.93 \\
\bottomrule
\end{tabular}
\end{table}

\subsection{The paired candidate-weight diagnostic}
\label{sec:real_omission}

On the same test draws, we compare Oracle-WT with exact fixed-tilt
membership through \eqref{eq:omitted_weight_diagnostic}.
Figure~\ref{fig:real_omission} shows the paired increase in coverage from
including the candidate weight. It reaches $2.50$ percentage points on
AqSolDB and $6.30$ on California housing at $-1.5$, but only $0.36$ on
superconductivity at that tilt. The first two increases exceed the
corresponding Oracle-WT shortfalls below nominal, $0.84$ and $2.93$ points,
because exact coverage is conservative in those conditions. The diagnostic
therefore measures a difference between rules, not a deficit from $.90$.

The effective calibration size at the true tilt is
\begin{equation}
 n_w^{\rm eff}(\boldsymbol\beta)
 =\frac{\{\sum_iW_i(\boldsymbol\beta)\}^2}
        {\sum_iW_i(\boldsymbol\beta)^2}.
 \label{eq:ess}
\end{equation}
AqSolDB and California housing have similar ESS at $-1.5$, approximately
112 and 123, despite their different omission effects. Weight concentration
alone does not determine the diagnostic, which also depends on the
candidate's position among the weighted scores.

\begin{figure}[t]
\centering
\begin{tikzpicture}
\begin{axis}[
 width=.84\linewidth,height=7cm,
 xmin=-.15,xmax=7.1,ymin=-.65,ymax=11.25,
 xtick={0,1,2,3,4,5,6,7},
 ytick={10,9,8,6,5,4,2,1,0},
 yticklabels={AqSolDB $-1.5$,AqSolDB $-1.0$,AqSolDB $+1.0$,
 California $-1.5$,California $-1.0$,California $+1.0$,
 Supercond.\ $-1.5$,Supercond.\ $-1.0$,Supercond.\ $+1.0$},
 xlabel={Exact coverage minus Oracle-WT (percentage points)},
 tick label style={font=\footnotesize},label style={font=\small},
 axis x line*=bottom,axis y line*=left,ytick style={draw=none},
 xmajorgrids,grid style={gray!15},clip=false]
\node[font=\footnotesize\bfseries,anchor=south] at (axis cs:7.65,10.6) {ESS};
\addplot[only marks,mark=*,mark size=2pt,blue!65!black,
 error bars/.cd,x dir=both,x explicit,error bar style={line width=.7pt}]
 coordinates {(2.500000000,10) +- (0.306359156,0)};
\node[anchor=west,font=\scriptsize] at (axis cs:2.926359156,10) {2.50};
\node[font=\footnotesize] at (axis cs:7.65,10) {112};
\addplot[only marks,mark=*,mark size=2pt,blue!65!black,
 error bars/.cd,x dir=both,x explicit,error bar style={line width=.7pt}]
 coordinates {(0.562500000,9) +- (0.071214495,0)};
\node[anchor=west,font=\scriptsize] at (axis cs:0.753714495,9) {0.56};
\node[font=\footnotesize] at (axis cs:7.65,9) {304};
\addplot[only marks,mark=*,mark size=2pt,blue!65!black,
 error bars/.cd,x dir=both,x explicit,error bar style={line width=.7pt}]
 coordinates {(0.212500000,8) +- (0.034559884,0)};
\node[anchor=west,font=\scriptsize] at (axis cs:0.367059884,8) {0.21};
\node[font=\footnotesize] at (axis cs:7.65,8) {589};
\addplot[only marks,mark=*,mark size=2pt,blue!65!black,
 error bars/.cd,x dir=both,x explicit,error bar style={line width=.7pt}]
 coordinates {(6.295833333,6) +- (0.375752149,0)};
\node[anchor=west,font=\scriptsize] at (axis cs:6.791585482,6) {6.30};
\node[font=\footnotesize] at (axis cs:7.65,6) {123};
\addplot[only marks,mark=*,mark size=2pt,blue!65!black,
 error bars/.cd,x dir=both,x explicit,error bar style={line width=.7pt}]
 coordinates {(0.404166667,5) +- (0.060626341,0)};
\node[anchor=west,font=\scriptsize] at (axis cs:0.584793007,5) {0.40};
\node[font=\footnotesize] at (axis cs:7.65,5) {332};
\addplot[only marks,mark=*,mark size=2pt,blue!65!black,
 error bars/.cd,x dir=both,x explicit,error bar style={line width=.7pt}]
 coordinates {(0.141666667,4) +- (0.026817195,0)};
\node[anchor=west,font=\scriptsize] at (axis cs:0.288483862,4) {0.14};
\node[font=\footnotesize] at (axis cs:7.65,4) {533};
\addplot[only marks,mark=*,mark size=2pt,blue!65!black,
 error bars/.cd,x dir=both,x explicit,error bar style={line width=.7pt}]
 coordinates {(0.362500000,2) +- (0.052623610,0)};
\node[anchor=west,font=\scriptsize] at (axis cs:0.535123610,2) {0.36};
\node[font=\footnotesize] at (axis cs:7.65,2) {269};
\addplot[only marks,mark=*,mark size=2pt,blue!65!black,
 error bars/.cd,x dir=both,x explicit,error bar style={line width=.7pt}]
 coordinates {(0.162500000,1) +- (0.035566938,0)};
\node[anchor=west,font=\scriptsize] at (axis cs:0.318066938,1) {0.16};
\node[font=\footnotesize] at (axis cs:7.65,1) {458};
\addplot[only marks,mark=*,mark size=2pt,blue!65!black,
 error bars/.cd,x dir=both,x explicit,error bar style={line width=.7pt}]
 coordinates {(0.062500000,0) +- (0.017451507,0)};
\node[anchor=west,font=\scriptsize] at (axis cs:0.199951507,0) {0.06};
\node[font=\footnotesize] at (axis cs:7.65,0) {607};
\end{axis}
\end{tikzpicture}
\caption{Candidate-weight omission on the real-data trials. Points show
the paired coverage increase from including the candidate weight, with
bars of one Monte Carlo standard error across 60 trials and labels in
percentage points. ESS is the mean effective calibration size at the true
tilt. The diagnostic uses \eqref{eq:omitted_weight_diagnostic} on shared
test points and retains ties. Errors condition on each observed dataset.}
\label{fig:real_omission}
\end{figure}

\paragraph{Scope of the real-data evidence.}
\label{sec:real_scope}
The comparisons concern fixed empirical reservoirs and the stated
estimators. Their standard errors describe repeated sampling within those
datasets, not uncertainty from collecting new datasets. Supplied budgets
that contain the truth and high empirical coverage do not establish a
general guarantee for the practical sets, or validate the label-shift
assumption in an uncontrolled deployment.

\subsection{Synthetic isolation of sampling error and predictive bias}
\label{sec:linear_exp}

The first synthetic study separates target-input sampling variability
from systematic predictive error. Source data follow
\begin{equation}
 X\sim\N(0,I_{15}),\qquad
 Y=X^\top\theta^\star+\varepsilon,\qquad
 \theta^\star=\mathbf1/\sqrt{15},\quad
 \varepsilon\sim\N(0,0.3),
 \label{eq:dgp}
\end{equation}
where $X$ and $\varepsilon$ are independent. Consequently, $Y\sim\N(0,1.3)$.
Gaussian Bayesian linear regression uses prior
$\N(0,1.5I)$ and known noise variance $0.3$. Each of 60 trials has 400
training, 300 calibration, and 200 test observations. The true tilt is
$\beta^\star=1$, and JTS-SCB uses 61 equally spaced points on $[0,1]$.
Targets are drawn exactly with independent
$X_t\sim\N(\beta^\star\theta^\star,I_{15})$ and
$\varepsilon_t\sim\N(0.3\beta^\star,0.3)$, then
$Y_t=X_t^\top\theta^\star+\varepsilon_t$.

Let $a$ and $\bar v$ be average predictive means and variances at the
unlabeled target inputs. Matching the tilted predictive mean
$a+\beta\bar v$ to the known source-tilted mean $1.3\beta$ gives
\begin{equation}
 \widehat\beta_{\rm TPS}=\frac{a}{1.3-\bar v}.
 \label{eq:tps_analytic}
\end{equation}
The denominator is positive in every trial. TPS uses this estimate,
whereas TPS-box projects it onto $[0,1]$. TPS-boot resamples the inputs
200 times and uses the 5th and 95th percentiles of the refitted estimates
as a range. It reports the interval hull over 61 equally spaced tilts,
each with its own threshold. This range propagates sampling variability
without a guaranteed family-containment property. Score-only uses the
same fixed grid as JTS-SCB with unit weights and a common-threshold hull.

\paragraph{Sampling variability and predictive bias.}
With 300 target inputs and a correctly specified predictor, TPS matches
Oracle-WT at coverage $.898$ and width $1.83$
(Table~\ref{tab:linear_reliability}). With only ten inputs, TPS-boot
covers $.930$ at width $2.10$, below the JTS-SCB width of $2.27$.
Resampling can thus be an efficient response to sampling variability alone.
These two conditions share training, calibration, and test draws, which
explains their identical oracle, JTS-SCB, and score-only results.

\begin{table}[t]
\centering
\caption{Sampling variability and predictive bias at $\beta^\star=1$,
$\mathcal B=[0,1]$, nominal coverage $.90$, and 60 paired trials. Entries
are coverage/mean width. TPS-box and JTS share the parameter bounds, whose
upper endpoint is the true tilt.
Score-only uses unit calibration weights and the common-threshold hull.
Both bias rows use 300 target inputs. Bias is added to all predictive
means after source fitting.}
\label{tab:linear_reliability}
\footnotesize
\setlength{\tabcolsep}{3pt}
\begin{tabular}{lrrrrrr}
\toprule
Condition & Oracle-WT & TPS & TPS-box & TPS-boot & JTS-SCB & Score-only \\
\midrule
Well specified, $n_{\rm tg}=300$ & .898/1.83 & .898/1.83 & .899/1.83 & .907/1.89 & .951/2.27 & .962/2.39 \\
Well specified, $n_{\rm tg}=10$ & .898/1.83 & .886/1.82 & .897/1.85 & .930/2.10 & .951/2.27 & .962/2.39 \\
Mean bias $+0.5$ & .894/2.43 & .821/2.41 & .894/2.43 & .838/2.48 & .968/2.84 & .989/3.33 \\
Mean bias $-0.5$ & .893/2.40 & .844/2.43 & .844/2.43 & .858/2.51 & .924/2.89 & .904/2.75 \\
\bottomrule
\end{tabular}
\end{table}

The bias conditions use 300 target inputs and add $b=\pm0.5$ to every
predictive mean after training, including calibration, pseudo-response,
and test means. Variances remain unchanged. By \eqref{eq:tps_analytic},
the TPS estimate moves by $b/(1.3-\bar v)$. A bootstrap can concentrate
around this displaced value without removing the bias. TPS and TPS-boot
undercover for both signs. TPS-box recovers oracle coverage under positive
bias by projecting onto the true endpoint, but still covers only $.844$
under negative bias, when its estimate remains inside the box. JTS-SCB
covers $.968$ and $.924$ at widths $2.84$ and $2.89$. Protection in these
specified conditions does not imply robustness to arbitrary model error.

\paragraph{The role of the bounds and the weights.}
An interior-truth control changes only the TPS-box bounds to $[0,2]$ on
the same 60 trials with 300 target inputs. Under positive bias, all
unconstrained estimates lie between $1.366$ and $1.695$, so none is
projected. TPS-box then matches TPS at coverage $.821$ and width $2.41$,
instead of $.894$ under $[0,1]$. The paired coverage difference is $.0731$
(SE $.0061$). The earlier improvement thus comes from the endpoint
constraint, not merely from containing the truth. JTS is unchanged in
this control. By Proposition~\ref{prop:monotonicity}, a $[0,2]$ grid
containing its original 61 points would give nested JTS-SCB sets and
hulls, hence coverage at least $.968$ on these test draws. Its additional
width has not been measured.

Under negative bias, the score-only individual-threshold union covers
$.801$ at width $2.49$, compared with $.895$ and $2.74$ for JTS-Union.
With common thresholds, the coverages are $.904$ and $.924$. Holding
that rule fixed, tilt-dependent weights add $.0196$ coverage
(paired SE $.0030$) at additional width. Threshold enlargement itself
accounts for much of the gain over the score-only union. Indeed, the
score-only common-threshold hull exceeds nominal empirical coverage in
all the conditions in Table~\ref{tab:linear_reliability}.

\subsection{Quadratic shifts with limited target information}
\label{sec:quadratic_exp}

The second study isolates estimation instability when the ratio includes
a dispersion tilt,
\begin{equation}
 w(y;\beta_1,\beta_2)\propto e^{\beta_1y+\beta_2y^2},\qquad
 (\beta_1^\star,\beta_2^\star)=(0.30,0.15).
 \label{eq:quadratic_truth}
\end{equation}
The source model adds a symmetric residual component to \eqref{eq:dgp},
\begin{equation}
 Y=X^\top\theta^\star+K\delta+\varepsilon,\qquad
 \Prob(K=1)=\Prob(K=-1)=\tfrac12,\qquad\delta=0.8,
 \label{eq:bimodal_dgp}
\end{equation}
with $K$ independent of $X$ and $\varepsilon$. The working Gaussian model
uses residual variance $0.3+0.8^2=0.94$, matching the variance but not the
bimodal shape. The exact tilted-mixture sampler is given in
Appendix~\ref{app:synthetic_details}. Each of 200 trials has 400 training,
300 calibration, and 200 test observations. The target-input samples of
sizes 10 and 30 are nested within the size-300 sample, with the predictor,
calibration sample, and test responses shared across sizes.

SPS-Q uses 20 source predictive draws per input. TPS-Q iterates tilted
draws and quadratic logistic ratio fits. Both allow any linear coefficient
and $\beta_2\le.30$, ensuring normalizability for the evaluated tilts.
TPS-Q-box restricts the coefficients to $[0,.30]\times[0,.15]$, also
supplied to JTS-Q, the JTS-SCB hull on a $7\times7$ grid. The truth lies
at the upper corner. Across sample sizes, $81$--$97\%$ of constrained
fits reach some boundary, but only $19.5$--$24\%$ reach that corner
within $10^{-4}$ per coordinate. Fixed normal draws are shared between
TPS versions and across nested samples and iterations.

\begin{figure}[t]
\centering
\definecolor{qBlue}{RGB}{0,114,178}
\definecolor{qOrange}{RGB}{213,94,0}
\definecolor{qGreen}{RGB}{0,130,105}
\definecolor{qPurple}{RGB}{154,78,148}
\begin{tikzpicture}
\begin{groupplot}[
 group style={group size=2 by 1,horizontal sep=1.15cm},
 width=.46\linewidth,height=5.5cm,
 xmode=log,xmin=8,xmax=380,xtick={10,30,300},xticklabels={10,30,300},
 xlabel={Target-input sample size},
 tick label style={font=\footnotesize},label style={font=\small},
 title style={font=\small},
 axis lines=left,ymajorgrids,grid style={gray!15},
 every axis plot/.append style={line width=.8pt},
 legend style={font=\footnotesize,draw=none,column sep=8pt},
 legend columns=3]
\nextgroupplot[title={(a) Coverage},ylabel={Coverage},ymin=.75,ymax=1.005,
 legend to name=quadraticlegend,
 ytick={.75,.80,.85,.90,.95,1.00},
 yticklabel style={/pgf/number format/fixed,/pgf/number format/precision=2}]
\addplot[gray,dashed,forget plot] coordinates {(8,.90) (380,.90)};
\addplot[black,densely dashed,mark=o,mark options={fill=white},mark size=2.5pt,
 error bars/.cd,y dir=both,y explicit,error bar style={solid,line width=.6pt}]
 coordinates {(10,0.897700000) +- (0,0.002045905) (30,0.897700000) +- (0,0.002045905) (300,0.897700000) +- (0,0.002045905)};
\addlegendentry{Oracle-WT}
\addplot[qGreen,dashdotted,mark=triangle*,mark size=2.5pt,
 error bars/.cd,y dir=both,y explicit,error bar style={solid,line width=.6pt}]
 coordinates {(10,0.857650000) +- (0,0.004329679) (30,0.884450000) +- (0,0.002703708) (300,0.891900000) +- (0,0.002266359)};
\addlegendentry{SPS-Q}
\addplot[qOrange,solid,mark=*,mark size=2pt,
 error bars/.cd,y dir=both,y explicit,error bar style={solid,line width=.6pt}]
 coordinates {(10,0.790850000) +- (0,0.010663545) (30,0.869450000) +- (0,0.004047114) (300,0.891975000) +- (0,0.002244451)};
\addlegendentry{TPS-Q}
\addplot[qPurple,densely dotted,mark=diamond*,mark size=3pt,
 error bars/.cd,y dir=both,y explicit,error bar style={solid,line width=.6pt}]
 coordinates {(10,0.892950000) +- (0,0.002770127) (30,0.898400000) +- (0,0.001969861) (300,0.899800000) +- (0,0.001937737)};
\addlegendentry{TPS-Q-box}
\addplot[qBlue,solid,mark=square*,mark size=2pt,
 error bars/.cd,y dir=both,y explicit,error bar style={solid,line width=.6pt}]
 coordinates {(10,0.976575000) +- (0,0.000928727) (30,0.976575000) +- (0,0.000928727) (300,0.976575000) +- (0,0.000928727)};
\addlegendentry{JTS-Q}
\nextgroupplot[title={(b) Width},ylabel={Mean width},ymin=3.05,ymax=4.5,ytick={3.2,3.6,4.0,4.4}]
\addplot[black,densely dashed,mark=o,mark options={fill=white},mark size=2.5pt,
 error bars/.cd,y dir=both,y explicit,error bar style={solid,line width=.6pt}]
 coordinates {(10,3.432938038) +- (0,0.018737016) (30,3.432938038) +- (0,0.018737016) (300,3.432938038) +- (0,0.018737016)};
\addplot[qGreen,dashdotted,mark=triangle*,mark size=2.5pt,
 error bars/.cd,y dir=both,y explicit,error bar style={solid,line width=.6pt}]
 coordinates {(10,3.197079488) +- (0,0.021880379) (30,3.212405079) +- (0,0.014071820) (300,3.205710449) +- (0,0.011646551)};
\addplot[qOrange,solid,mark=*,mark size=2pt,
 error bars/.cd,y dir=both,y explicit,error bar style={solid,line width=.6pt}]
 coordinates {(10,3.670256694) +- (0,0.098678124) (30,3.617680684) +- (0,0.054364274) (300,3.521196567) +- (0,0.029123428)};
\addplot[qPurple,densely dotted,mark=diamond*,mark size=3pt,
 error bars/.cd,y dir=both,y explicit,error bar style={solid,line width=.6pt}]
 coordinates {(10,3.352036133) +- (0,0.018526650) (30,3.373461518) +- (0,0.018389502) (300,3.401565646) +- (0,0.017689649)};
\addplot[qBlue,solid,mark=square*,mark size=2pt,
 error bars/.cd,y dir=both,y explicit,error bar style={solid,line width=.6pt}]
 coordinates {(10,4.323867755) +- (0,0.014825942) (30,4.323867755) +- (0,0.014825942) (300,4.323867755) +- (0,0.014825942)};
\end{groupplot}
\node[anchor=north,yshift=-1.15cm]
 at ($(group c1r1.south)!0.5!(group c2r1.south)$)
 {\pgfplotslegendfromname{quadraticlegend}};
\end{tikzpicture}
\caption{Quadratic shifts with limited target inputs, 200 paired trials.
Points show mean coverage and width, with bars of one Monte Carlo standard
error. Width is on the original response scale. The dashed horizontal line
marks nominal coverage $.90$, and the sample-size axis is logarithmic.
TPS-Q-box and JTS-Q share $[0,.30]\times[0,.15]$, whose upper corner is
the true tilt. SPS-Q and TPS-Q use the broader range in the text.
Oracle-WT and JTS-Q do not use target inputs and are identical across sizes.}
\label{fig:quadratic_ntg}
\end{figure}

Figure~\ref{fig:quadratic_ntg} separates the coverage and width trends.
As target-input size falls from 300 to 30 to 10, broad-range TPS-Q
coverage declines from $.892$ to $.869$ to $.791$. At size 30 it is
$.0150$ below SPS-Q (paired SE $.0037$). TPS-Q-box instead covers
$.900$, $.898$, and $.893$, close to the oracle's $.898$. At size 10 its
paired difference from the oracle is $-.0048$ (SE $.0024$). JTS-Q covers
$.977$ at width $4.324$, about $26\%$ wider than Oracle-WT and $29\%$
wider than TPS-Q-box in that condition.

The broad-range TPS-Q root mean squared error for $\beta_2$ rises from
$.038$ to $.098$ to $.350$, against $.015$, $.036$, and $.063$ for
TPS-Q-box. Since
\begin{equation}
 \log\widehat w(y)-\log w(y)
 =(\widehat\beta_1-\beta_1^\star)y
  +(\widehat\beta_2-\beta_2^\star)y^2+\text{constant},
 \label{eq:log_weight_error}
\end{equation}
quadratic estimation error can strongly affect tail weights. The results
show instability for the stated broad-range estimator and substantial
improvement from a favorable box. They do not establish that quadratic
TPS is intrinsically unreliable, or that a general containing box gives
the same improvement. No quadratic interior-truth comparison is reported.
The linear control and the positive superconductivity shift already show
that bounded TPS can remain unreliable when the truth is inside the range.

\FloatBarrier

\section{Discussion}
\label{sec:discussion}

JTS-CB treats uncertainty about continuous label shift as a joint calibration
problem. Each candidate tilt determines both a Bayesian score and its
importance weights. The smallest minimizer of the corresponding weighted
quantile-regression loss gives that candidate's calibration threshold.
Taking the supremum of these selected quantiles yields the smallest common
threshold satisfying every candidate's weighted empirical calibration
constraint. The prediction union then preserves the different score
geometries represented by the family.

JTS-CB carries the prespecified tilt family through both calibration and
prediction. A plug-in procedure instead selects one tilt estimate before
calibration. JTS-SCB summarizes the tilt-specific calibration requirements
by a common threshold and admits an explicit Gaussian sensitivity envelope.
JTS-Union retains the individual thresholds and produces sets contained in
that envelope. Its width reductions in the real-data study support its use
when numerical unioning is acceptable. The common-threshold construction
remains useful for its direct optimization interpretation and transparent
decomposition of Gaussian interval width.

Both practical procedures use calibration-only weights by design. Their
empirical calibration constraints therefore do not establish a general
nominal coverage guarantee. The paired real-data diagnostic shows that
omitting the candidate weight can reduce coverage even when the tilt is
known. Under the stated sampling and family-containment assumptions, exact
candidate-weighted inversion supplies the inherited finite-sample marginal
guarantee. Nonzero linear and positive-quadratic candidates can nevertheless
force acceptance of extreme responses. Compact families of strictly
negative-quadratic tilts instead admit bounded exact unions under Gaussian
scores. For the practical procedures, the localized correction provides a
coverage bound while retaining the same inner quantile problems. It requires
a justified target tail-mass bound and can incur substantial width or
fallback costs.

The experiments identify benefits and limitations of both coupling and
sensitivity. JTS-Union gives narrower sets with higher observed coverage
than varying weights alone in the negative real-data shifts, but adds width
in the positive shifts. Coupling therefore has no uniform empirical
efficiency advantage. Reliable TPS fits approach oracle coverage with
substantially narrower sets than the sensitivity procedures. In the tested
predictive-bias conditions, JTS-SCB maintains empirical coverage above
nominal at an additional width cost. In the quadratic experiment,
TPS-Q-box approaches oracle coverage with narrower sets than JTS-Q,
although the true tilt lies at a corner of the supplied box. The linear
interior-truth control shows why the benefit of such bounds should not be
generalized. The wrong-budget comparison further demonstrates that
excluding relevant tilts can undermine the protection supplied by
sensitivity analysis.

Practical use requires a defensible specification of the tilt family.
The assumption that the true density ratio is represented by a tilt in
$\mathcal B$ is distinct from monotonicity as the family expands. A larger
family can remain structurally wrong, particularly when the label-shift
assumption fails. The experimental budgets are supplied conditions, and
their performance does not establish how to choose a suitable family from
unlabeled target inputs. A finite computational grid also does not by
itself justify coverage guarantees for between-grid tilts. Finally, the
population tail bound required by the correction cannot be justified
merely by substituting empirical source averages. Weight concentration,
correction feasibility, coverage, and width describe different aspects of
the procedure and should be assessed separately.

Ratio clipping generally changes the target represented by the weights.
A coverage guarantee for the original target therefore requires explicit
control of the discrepancy introduced by clipping. The present development
keeps the original target and focuses on the joint calibration formulation,
its exact benchmark, and the localized correction. Further work should
develop independently justified sensitivity families and more efficient
certified optimization over them. The practical value of JTS-CB depends
on whether the resulting coverage and width justify carrying that
uncertainty through prediction.

\begingroup
\bibliographystyle{plainnat}
\bibliography{sjc}
\endgroup

\clearpage
\appendix

\section{Optimization and prediction-set properties}
\label{app:optimization}

\subsection{Weighted quantile regression and robust feasibility}

\begin{proof}[Proof of Lemma~\ref{lem:weighted_pinball_quantile}]
Fix a tilt and suppress its subscript. The normalized weights sum to one.
For $\rho_\tau(S_i-q)$, the derivative in $q$ is $-\tau$ when $q<S_i$
and $1-\tau$ when $q>S_i$, while its subgradient at $q=S_i$ is
$[-\tau,1-\tau]$. Summing the weighted contributions gives
\begin{equation}
 \partial L(q)
 =\bigl[\widehat F(q^-)-\tau,\ \widehat F(q)-\tau\bigr].
 \label{eq:pinball_subgradient}
\end{equation}
Convexity makes $0\in\partial L(q)$ necessary and sufficient for
minimization, proving \eqref{eq:weighted_quantile_condition}. The lower
weighted quantile $\widehat q$ is the smallest point at which
$\widehat F$ reaches $\tau$. It satisfies the subgradient condition,
whereas every $q<\widehat q$ has $\widehat F(q)<\tau$ and cannot minimize
$L$. Thus $\widehat q=\min\operatorname*{arg\,min}L$. Positive weights,
finite scores, and $\tau\in(0,1)$ ensure that this minimum exists.
This argument also covers ties and flat intervals of minimizers.
\end{proof}

\paragraph{Equivalence of the two optimization formulations.}
For each fixed tilt, the lower empirical quantile satisfies
$\widehat q_{\boldsymbol\beta}\le q$ if and only if
$\widehat F_{\boldsymbol\beta}(q)\ge\tau$. Therefore
\begin{align*}
 \sup_{\boldsymbol\beta\in\mathcal B}
   [\tau-\widehat F_{\boldsymbol\beta}(q)]\le0
 &\quad\Longleftrightarrow\quad
 \widehat q_{\boldsymbol\beta}\le q
       \quad\hbox{for every }\boldsymbol\beta\in\mathcal B\\
 &\quad\Longleftrightarrow\quad
 \sup_{\boldsymbol\beta\in\mathcal B}\widehat q_{\boldsymbol\beta}\le q.
\end{align*}
If the last supremum is finite, the feasible thresholds form the closed
ray beginning at that value. If it is infinite, no finite threshold is
feasible. Taking the infimum, with $\inf\varnothing=+\infty$, proves
\eqref{eq:jts_minmax_feasibility}. No maximizing tilt need exist for this
equivalence. It uses only the defining quantile inequalities.

\paragraph{Linear-programming form of the inner problem.}
For a fixed tilt, the same weighted quantile-regression problem is
\begin{equation}
 \begin{aligned}
 \underset{q\in\R,\,u,v\ge0}{\operatorname{minimize}}\quad
 &\sum_{i=1}^n\overline W_i(\boldsymbol\beta)
       \{\tau u_i+(1-\tau)v_i\}\\
 \text{subject to}\quad
 &S_i(\boldsymbol\beta)-q=u_i-v_i,
       \qquad i=1,\ldots,n.
 \end{aligned}
 \label{eq:pinball_linear_program}
\end{equation}
At any fixed $q$, the optimal slacks are
$u_i=(S_i(\boldsymbol\beta)-q)_+$ and
$v_i=(q-S_i(\boldsymbol\beta))_+$. Selecting the smallest optimal $q$
therefore recovers \eqref{eq:jts_quantile_regression_lower}. The experiments
solve this scalar quantile problem by sorting and use a finite outer grid.
They do not use a generic saddle-point solver or exchange the order of
optimization. 

\subsection{Properties of the practical prediction sets}

\begin{proof}[Proof of Proposition~\ref{prop:monotonicity}]
For the tilt-specific union, every tilt in $\mathcal B_1$ also belongs to
$\mathcal B_2$, so unioning over $\mathcal B_2$ can only add candidate
responses. For the common-threshold construction,
\begin{equation*}
  \widehat\theta_1
  =\sup_{\boldsymbol\beta\in\mathcal B_1}
       \widehat q_{\boldsymbol\beta}
  \le
  \sup_{\boldsymbol\beta\in\mathcal B_2}
       \widehat q_{\boldsymbol\beta}
  =\widehat\theta_2.
\end{equation*}
If a response is accepted under $\mathcal B_1$, the same tilt is present under
$\mathcal B_2$ and its acceptance threshold is no smaller. This proves both
pathwise inclusions, and taking probabilities gives coverage monotonicity.
\end{proof}

\begin{proof}[Proof of Proposition~\ref{prop:gaussian_interval}]
Every empirical quantile is at most its largest calibration score.
Continuity on the compact interval therefore gives
$\widehat\theta_{\rm JTS}\le
\max_i\sup_{\beta\in\mathcal B}S_i(\beta)<\infty$.
For fixed $\beta$, the inequality
$S_\beta(x,y)\le\widehat\theta_{\rm JTS}$ is equivalent, when $r(x)\ge0$, to
\begin{equation*}
  |y-\mub(x)-\beta\sigb^2(x)|\le h(x).
\end{equation*}
It therefore defines an interval with center
$\mub(x)+\beta\sigb^2(x)$ and half-width $h(x)$ independent of $\beta$.
Unioning these continuously moving intervals over
$[\beta_{\rm lo},\beta_{\rm hi}]$ gives
\eqref{eq:closed_form_interval}, and subtracting its endpoints gives
\eqref{eq:closed_form_width}. If $r(x)<0$, the threshold is below the minimum
Gaussian score and the set is empty.
\end{proof}

\paragraph{Why a worst-case pinball-loss value is different.}
Consider two candidate score distributions, each with one unit-weight
observation, at scores 0 and 1 respectively, and take $\tau=0.9$.
Their lower quantile solutions are 0 and 1, so the joint common threshold
is 1. On $q\in[0,1]$, the two pinball losses are $0.1q$ and
$0.9(1-q)$. Minimizing their maximum gives $q=0.9$, which fails the
$0.9$ empirical mass requirement for the candidate with score 1.
Thus minimizing $\sup_{\boldsymbol\beta}L_{\boldsymbol\beta}(q)$ does
not generally produce the largest selected quantile. The counterexample
concerns the algebra of the optimization, without requiring a specific
predictive model.

\paragraph{Gaussian predictive identity.}
For source mean $\mu$ and variance $v>0$, the logarithm of the unnormalized
quadratic predictive tilt is
\[
 -\frac{(y-\mu)^2}{2v}+\beta_1y+\beta_2y^2
 =-\frac{1-2\beta_2v}{2v}
   \left(y-\frac{\mu+\beta_1v}{1-2\beta_2v}\right)^2
   +\text{constant}.
\]
It is integrable exactly when $1-2\beta_2v>0$, giving
\eqref{eq:quadratic_predictive}. The linear formulas follow by setting
$\beta_2=0$. This completion of the square also establishes continuity of
the mean and variance on compact subsets of the admissible parameter region.

\section{Weighted coverage and tail geometry}
\label{app:validity}
\subsection{Weighted-rank argument}

\begin{proof}[Proof of Corollary~\ref{cor:exact_coverage}]
Condition on $\HH$. At the true tilt, the joint density relative to
$P_s^{n+1}$ is $w(Y_{n+1};\boldsymbol\beta^\star)$. Conditional on the
unordered augmented sample, weighted exchangeability assigns probability
$W_j/\sum_{k=1}^{n+1}W_k$ to observation $j$ being the target point
\citep[Theorem~2]{TibshiraniR2019neurips}. All scores
$S_j=S_{\boldsymbol\beta^\star}(X_j,Y_j)$ use the same fixed function
given $\HH$. If $j$ is the target index, its weighted upper-tail rank is
\[
 p_j=\frac{\sum_{k=1}^{n+1}W_k\1\{S_k\ge S_j\}}
            {\sum_{k=1}^{n+1}W_k}.
\]
Order the distinct score values from largest to smallest and accumulate
their weight, retaining entire tie groups. The total normalized weight of
groups with $p_j\le\alpha$ is at most $\alpha$. Thus
\[
 \Prob\{\pi_{\boldsymbol\beta^\star}(Y_{\rm test})\le\alpha\mid\HH\}
 \le\alpha,
\]
including when scores tie. The exact set uses the complementary strict
acceptance rule, so its coverage is at least $1-\alpha$. The uncertainty-set
union contains that true-tilt set, proving \eqref{eq:exact_coverage}.
Averaging over $\HH$ gives the marginal bound. This is the standard
weighted-rank proof followed by set inclusion.
\end{proof}

\subsection{Uniform tail argument}

\begin{proof}[Proof of Proposition~\ref{prop:tails}]
For a nonzero linear tilt, $c_\beta e^{\beta y}$ exceeds the fixed finite
threshold $K_\beta$ on an entire half-line. Equation~\eqref{eq:weight_forces_acceptance}
proves part (i), without any requirement on the true tilt or score growth.
For a fixed negative-quadratic tilt, the candidate weight tends to zero
in both tails while the proper Gaussian score tends to infinity. The
tail criterion therefore rejects sufficiently extreme responses. For a
positive-quadratic tilt the weight diverges in both tails and
\eqref{eq:weight_forces_acceptance} implies eventual acceptance. A nonzero
linear tilt gives the corresponding one-sided conclusion.

At the no-shift point all weights equal one. Every candidate then has
$\pi(y)\ge1/(n+1)$, with equality whenever its score exceeds all calibration
scores. Gaussian coercivity gives a bounded set when $\alpha\ge1/(n+1)$.
Equality is rejected by the strict acceptance rule. If
$\alpha<1/(n+1)$ every response is accepted.

For the uniform statement, let $\mathcal B_Q$ be compact and contained in
$\{\beta_2<0\}$. There exist $\varepsilon>0$ and $B_1<\infty$ such that
$\beta_2\le-\varepsilon$ and $|\beta_1|\le B_1$ throughout the set.
Use unnormalized weights
$u_{\boldsymbol\beta}(y)=\exp(\beta_1y+\beta_2y^2)$, which produce the
same $p$-values. Then
\[
 \sup_{\boldsymbol\beta\in\mathcal B_Q}u_{\boldsymbol\beta}(y)
 \le\exp(B_1|y|-\varepsilon y^2)\longrightarrow0
 \quad\text{as }|y|\to\infty.
\]
For a realized finite calibration sample, positivity, continuity, and
compactness imply
\[
 \inf_{\boldsymbol\beta\in\mathcal B_Q}
       \sum_{i=1}^n u_{\boldsymbol\beta}(Y_i)>0.
\]
The tilted Gaussian means and variances are continuous in the parameter.
At each fixed input their means are bounded and their variances are
bounded above and away from zero on $\mathcal B_Q$. In particular, for
constants $m_x<\infty$ and $0<v_-\le v_+<\infty$ bounding the absolute
means and variances at $x$, respectively,
\[
 \inf_{\boldsymbol\beta\in\mathcal B_Q}S_{\boldsymbol\beta}(x,y)
 \ge\frac{(|y|-m_x)_+^2}{2v_+}+\frac12\log(2\pi v_-)
 \longrightarrow\infty.
\]
The finite calibration score maximum is uniformly bounded by continuity. One finite cutoff satisfies the extreme-score
condition and makes the candidate weight no larger than
$\alpha/(1-\alpha)$ times the calibration weight sum for every tilt.
The entire union is consequently bounded.
Under Assumption~\ref{ass:validity}, coverage follows from
Corollary~\ref{cor:exact_coverage} if the true tilt belongs to the class.
\end{proof}

\section{Details of the localized correction}
\label{app:correction}

A candidate-weight bound gives a sufficient correction while preserving the
joint optimization. Choose a nonempty measurable response range $D$ using
only $\HH$, and suppose an $\HH$-measurable finite bound
$M_{\boldsymbol\beta}(D)\ge\sup_{y\in D}w(y;\boldsymbol\beta)$ is
available for each tilt. Retain the finite-score, positive-weight, and
measurability conditions above. Define
\begin{equation}
 \alpha'_{\boldsymbol\beta}(D)
 =\alpha-(1-\alpha)\frac{M_{\boldsymbol\beta}(D)}{C_{\boldsymbol\beta}},
 \qquad
 F_D=\{M_{\boldsymbol\beta}(D)<K_{\boldsymbol\beta}
                   \text{ for every }\boldsymbol\beta\in\mathcal B\}.
 \label{eq:deflated_level}
\end{equation}
On $F_D$, replace $\alpha$ by $\alpha'_{\boldsymbol\beta}(D)>0$ in each
inner quantile-regression problem. The corrected union retains those
individual solutions, and the corrected JTS-SCB takes their supremum,
\begin{align}
 \CC^{\rm cal}_{\rm union}(x;D)
 &=\bigcup_{\boldsymbol\beta\in\mathcal B}
   \{y:S_{\boldsymbol\beta}(x,y)\le
             \widehat q_{\boldsymbol\beta}(\alpha'_{\boldsymbol\beta}(D))\},
 \nonumber\\
 \widehat\theta^{\rm cal}_{\rm JTS}(D)
 &=\sup_{\boldsymbol\beta\in\mathcal B}
       \widehat q_{\boldsymbol\beta}(\alpha'_{\boldsymbol\beta}(D)),
 \nonumber\\
 \CC^{\rm cal}_{\rm JTS}(x;D)
 &=\bigcup_{\boldsymbol\beta\in\mathcal B}
   \{y:S_{\boldsymbol\beta}(x,y)\le\widehat\theta^{\rm cal}_{\rm JTS}(D)\}.
 \label{eq:corrected_sets}
\end{align}
On $F_D^c$, define the respective fallbacks as
\begin{equation}
 \CC^{\rm cal}_{\rm union}(x;D)=D\cup\CC^{\rm union}_{\rm JTS}(x),
 \qquad
 \CC^{\rm cal}_{\rm JTS}(x;D)=D\cup\CC^{\rm scb}_{\rm JTS}(x).
 \label{eq:corrected_fallback}
\end{equation}
The range bounds the candidate weight, but the output is not intersected
with $D$. All calibration samples, including fallback samples, are part
of the procedure.

For a probability measure $\nu$, let $Q_{1-\alpha}(\nu)$ denote its lower
quantile. On $F_D$ the correction has the identity
\begin{equation}
 \widehat q_{\boldsymbol\beta}(\alpha'_{\boldsymbol\beta}(D))
 =Q_{1-\alpha}\left(
 \frac{\sum_iW_i(\boldsymbol\beta)\delta_{S_i(\boldsymbol\beta)}
       +M_{\boldsymbol\beta}(D)\delta_{+\infty}}
      {C_{\boldsymbol\beta}+M_{\boldsymbol\beta}(D)}\right).
 \label{eq:augmented_correction}
\end{equation}
This standard augmented weighted quantile replaces the candidate weight
by an upper bound. It specializes \citet[Eq.~(7) and Algorithm~1]{JinY2023pnas}
with calibration bounds $\ell_i=u_i=W_i(\boldsymbol\beta)$ and candidate
bound $u_{n+1}=M_{\boldsymbol\beta}(D)$. The extension below combines
this device with the coupled family, localization, and fallback.

\begin{proposition}[Localized correction and containment]
\label{prop:correction}
Under the preceding conditions on $D$, the weight bounds, and the scores,
for every calibration sample and test input,
\begin{equation}
 \CC^{\rm exact}_{\mathcal B}(x)\cap D
 \subseteq\CC^{\rm cal}_{\rm union}(x;D)
 \subseteq\CC^{\rm cal}_{\rm JTS}(x;D),
 \label{eq:corrected_containment}
\end{equation}
and each corrected set contains its corresponding practical set everywhere.
Under Assumption~\ref{ass:validity}, if an $\HH$-measurable
$\bar\delta_D\in[0,1]$ satisfies
$\Prob\{Y_{\rm test}\notin D\mid\HH\}\le\bar\delta_D$ almost surely,
then either corrected set satisfies
\begin{equation}
 \Prob\{Y_{\rm test}\in\CC^{\rm cal}(X_{\rm test};D)\mid\HH\}
 \ge1-\alpha-\bar\delta_D.
 \label{eq:corrected_coverage}
\end{equation}
Both corrected sets are bounded at $x$, including on fallback samples,
if $D$ is bounded, the calibration score maximum is uniformly bounded
over $\mathcal B$, and the candidate scores are uniformly coercive,
\[
 \inf_{\boldsymbol\beta\in\mathcal B}S_{\boldsymbol\beta}(x,y)
 \longrightarrow\infty\qquad\text{as }|y|\longrightarrow\infty.
\]
\end{proposition}

The argument uses samplewise containment, not a fixed-level validity
theorem applied to the calibration-dependent $\alpha'_{\boldsymbol\beta}(D)$.
Coverage includes fallback samples and is not conditional on $F_D$.
The corrected common threshold solves the same minimum-threshold problem
with requirements $\widehat F_{\boldsymbol\beta}(q)\ge1-\alpha'_{\boldsymbol\beta}(D)$.

Under the tilt model, a known source label distribution supplies
\begin{equation}
 \Prob\{Y_{\rm test}\notin D\mid\HH\}
 \le\delta_{\mathcal B}(D)
 :=\sup_{\boldsymbol\beta\in\mathcal B}
 \frac{\E_{P_s^Y}[e^{\boldsymbol\beta^\top\phi(Y)}\1\{Y\notin D\}]}
      {\E_{P_s^Y}[e^{\boldsymbol\beta^\top\phi(Y)}]}.
 \label{eq:source_tail_mass}
\end{equation}
This bound applies to a known distribution, including a fully specified
empirical reservoir. Replacing its expectations by sample averages does
not certify a population bound. Enlarging $D$ can lower the tail penalty
while increasing its weight supremum, potentially triggering fallback.

\begin{proof}[Proof of Proposition~\ref{prop:correction}]
On $F_D$, take $y\in\CC^{\rm exact}_{\boldsymbol\beta}(x)\cap D$.
Equation~\eqref{eq:one_variable} gives
\begin{align*}
 A_{\boldsymbol\beta}(y)
 &>\alpha C_{\boldsymbol\beta}-(1-\alpha)W_{n+1}(y;\boldsymbol\beta)\\
 &\ge\alpha C_{\boldsymbol\beta}-(1-\alpha)M_{\boldsymbol\beta}(D)
 =\alpha'_{\boldsymbol\beta}(D)C_{\boldsymbol\beta}.
\end{align*}
The first identity in \eqref{eq:one_variable} therefore places $y$ in the
corrected fixed-tilt set. Union over tilts, then replace the individual
quantiles by their supremum, to obtain
\eqref{eq:corrected_containment}. On $F_D^c$ both fallbacks contain $D$,
and nesting of the practical sets gives nesting of the corrected ones.
Because $0<\alpha'_{\boldsymbol\beta}(D)<\alpha$ on $F_D$, corrected
quantiles are no smaller than the original quantiles. Each corrected set
therefore contains its practical counterpart on $F_D$, and the fallback
preserves this inclusion on $F_D^c$.

Containment holds on every sample, so
\begin{align*}
 \Prob\{Y\in\CC^{\rm cal}_{\rm union}(X;D)\mid\HH\}
 &\ge\Prob\{Y\in\CC^{\rm exact}_{\mathcal B}(X)\mid\HH\}
      -\Prob\{Y\notin D\mid\HH\}\\
 &\ge1-\alpha-\bar\delta_D.
\end{align*}
The larger corrected JTS-SCB set inherits the bound. Every positive-level
empirical quantile is at most the maximum calibration score. Uniform
boundedness of those maxima and uniform coercivity of the candidate scores
therefore bound both corrected sets on $F_D$, as well as the practical
sets in the fallback. If $D$ is bounded, the fallback union is bounded.
\end{proof}

\paragraph{Why the augmented-quantile identity holds.}
At a finite threshold $q$, the cumulative mass in
\eqref{eq:augmented_correction} is
\[
 \frac{C_{\boldsymbol\beta}}{C_{\boldsymbol\beta}+M_{\boldsymbol\beta}(D)}
 \widehat F_{\boldsymbol\beta}(q).
\]
Requiring it to reach $1-\alpha$ is equivalent to
\[
 \widehat F_{\boldsymbol\beta}(q)
 \ge(1-\alpha)\left\{1+
           \frac{M_{\boldsymbol\beta}(D)}{C_{\boldsymbol\beta}}\right\}
 =1-\alpha'_{\boldsymbol\beta}(D).
\]
This proves \eqref{eq:augmented_correction} on $F_D$. The augmented
quantile remains defined off that event. It equals $+\infty$ when
$M_{\boldsymbol\beta}(D)>K_{\boldsymbol\beta}$ and equals the maximum
calibration score at equality. The strict positive-level feasibility rule
is consequently conservative at the boundary. A zero-level extension
$\widehat q_{\boldsymbol\beta}(0)=\max_iS_i(\boldsymbol\beta)$ also
satisfies the containment argument when equality holds.

At the no-shift tilt, $W_i=M=1$ gives
$\alpha'=\alpha-(1-\alpha)/n$. When this is positive, the corrected
quantile has rank $\lceil(n+1)(1-\alpha)\rceil$, recovering the ordinary
split conformal adjustment. The same rank is $n$ at the equality boundary.

\paragraph{Feasibility and the role of the range.}
A common positive rescaling of a tilt's weights leaves
$M_{\boldsymbol\beta}(D)/C_{\boldsymbol\beta}$ unchanged. Thus the
correction can be computed using unnormalized weights, provided the
candidate bound uses that same scale. A certified upper bound may replace
a supremum, but an uncertified numerical maximum cannot supply the stated
guarantee.

For a bounded interval $D=[L,U]$, linear and quadratic exponential weights
have simple certified maxima. The linear maximum is at an endpoint. For
a quadratic exponent, evaluate both endpoints and its vertex
$-\beta_1/(2\beta_2)$ if $\beta_2<0$ and the vertex lies in $D$.
For negative-quadratic normalized weights, the global maximum is
\[
 M_{(\beta_1,\beta_2)}(\R)
 =\exp\{-\beta_1^2/(4\beta_2)-A_s(\beta_1,\beta_2)\}.
\]
A finite global upper bound does not ensure feasibility on every sample,
because the calibration weight sum may be small. A sufficient
deterministic condition is
$0<m_{\boldsymbol\beta}\le w(y;\boldsymbol\beta)\le M_{\boldsymbol\beta}$
with $M_{\boldsymbol\beta}/m_{\boldsymbol\beta}<\alpha n/(1-\alpha)$
for every tilt. Negative-quadratic weights have no positive global lower
bound, so their global correction can still require fallback.

Intersecting the output with $D$ would preserve the same coverage bound
but lose containment of practical responses outside $D$. For $D=\R$,
the tail penalty vanishes, growing-tail ratios have no finite global bound,
and fallback equals $\R$. Boundedness then requires feasibility.

\paragraph{Population and estimated tail bounds.}
In a population problem, estimating \eqref{eq:source_tail_mass} requires
an uncertainty bound accounting for estimation and any data-dependent
choice of $D$ or $\mathcal B$.

Averaging \eqref{eq:corrected_coverage} over $\HH$ gives marginal
coverage at least $1-\alpha-\E[\bar\delta_D]$. If an $\HH$-measurable
estimate $\widehat\delta_D\in[0,1]$ bounds the true conditional tail mass
only on an event of probability at least $1-\eta$, then
\[
 \Prob\{Y_{\rm test}\in\CC^{\rm cal}(X_{\rm test};D)\}
 \ge1-\alpha-\E[\widehat\delta_D]-\eta.
\]
Indeed, the true tail mass is at most
$\widehat\delta_D+\1\{\text{bound fails}\}$. Unlabeled target inputs
alone do not generally justify this missing tail bound.

\paragraph{Empirical cost of the correction.}
Table~\ref{tab:correction_cost} evaluates the common-threshold correction
on the same real-data trials with $D=[-3,3]$. Every average includes fallback
samples. The displayed lower bound is $0.90$ minus the mean of the
trial-specific worst-case tail probabilities over the 41-point family,
evaluated on each full empirical source reservoir. Each conditional bound
applies to its reservoir experiment. Their average is not a certified
bound for an unknown population.
The family-wide bound is $.593$ on California housing with the negative
budget, despite observed corrected coverage of $.929$. On AqSolDB with
the positive budget, feasibility is $.717$ and width rises from $2.10$ to
$4.12$. Smaller true-tilt tail probabilities would give oracle bounds,
not protection across the family. Width, fallback, and the tail penalty
can therefore limit the correction's practical value.

\begin{table}[t]
\centering
\caption{Cost of the common-threshold correction on the 60 real-data trials.
$D=[-3,3]$. Feasibility requires positive corrected levels at every grid
point. Corrected coverage/width includes all fallback samples. The final
column is the family-wide coverage lower bound averaged over trials.}
\label{tab:correction_cost}
\small
\setlength{\tabcolsep}{4pt}
\begin{tabular}{llrrrr}
\toprule
Data & $\beta^\star$ & Feasible & JTS-SCB width & Corrected & Lower bound \\
\midrule
AqSolDB & $-1.5$ & 1.000 & 2.86 & .951/3.06 & .678 \\
AqSolDB & $+1.0$ & .717 & 2.10 & .997/4.12 & .898 \\
California & $-1.5$ & 1.000 & 4.53 & .929/4.91 & .593 \\
California & $+1.0$ & .983 & 2.51 & .992/3.52 & .899 \\
Supercond. & $-1.5$ & 1.000 & 2.44 & .976/3.03 & .900 \\
Supercond. & $+1.0$ & .817 & 1.91 & .996/3.67 & .900 \\
\bottomrule
\end{tabular}
\end{table}
\FloatBarrier

\section{Experimental details and reproducibility}
\label{app:experiments}

\subsection{Real-data implementation}
\label{app:real_details}

California housing uses the StatLib archive, removes observations at the
censored response ceiling, and augments its eight features with rooms,
bedrooms, and population per household. Superconductivity uses 81
composition features. AqSolDB uses the 7488-row training split from
\texttt{maomlab/AqSolDB}, 17 descriptors from \texttt{MolWt} through
\texttt{BertzCT}, and a 128-dimensional signed hash of SMILES character
2--4-grams. It excludes \texttt{SD}, \texttt{Ocurrences},
\texttt{ClusterNo}, and \texttt{MolCount}. Response transformations are
specified in Section~\ref{sec:real}. Scaling uses training observations
only, and standardized numerical inputs are clipped to $[-8,8]$.

The ridge model adds 100 Gaussian random Fourier features with kernel
parameter equal to the inverse number of numerical features. Ridge
precision is one, with an effectively unpenalized intercept. Predictive
variance is the estimated residual variance times one plus posterior
leverage. No target responses enter fitting or tilt estimation.

For training labels $y_j$, define
$m_s(\beta)=\sum_jy_je^{\beta y_j}/\sum_je^{\beta y_j}$.
Given average target-input predictive mean $a$ and variance $\bar v$,
TPS inverts $m_s$ at $a+\beta\bar v$, clips unattainable moments to the
budget endpoints, and averages the new estimate with the old one using
equal weights. It starts at zero and stops at tolerance $10^{-6}$ or
200 iterations. All final iterates are used, including the single
nonconverged positive-shift California trial. Exact membership uses
\eqref{eq:one_variable} with ties retained. Interval comparisons use
tolerance $10^{-12}$.

\subsection{Synthetic implementation and randomization}
\label{app:synthetic_details}

The linear study uses the analytic TPS estimate and bootstrap construction
in Section~\ref{sec:linear_exp}. Its ten-input sample is nested within the
300-input sample. The interior-truth control changes only the TPS-box
upper bound from 1 to 2, leaving JTS on $[0,1]$.

For exact quadratic sampling, write $V=1.3$, $\mu_k=k\delta$, and
$d_{\boldsymbol\beta}=1-2\beta_2V>0$. Draw
\begin{align*}
 \Prob_t(K=k)&\propto\tfrac12\exp\!\left\{
 \frac{\beta_2\mu_k^2+\beta_1\mu_k+\tfrac12\beta_1^2V}
      {d_{\boldsymbol\beta}}\right\},\\
 Y\mid K=k&\sim\N\!\left(\frac{\mu_k+\beta_1V}{d_{\boldsymbol\beta}},
                          \frac{V}{d_{\boldsymbol\beta}}\right),\\
 X\mid Y,K&\sim\N\!\left\{\frac{\theta^\star}{V}(Y-K\delta),
                 I_{15}-\frac{\theta^\star\theta^{\star\top}}{V}\right\}.
\end{align*}
The common factor $d_{\boldsymbol\beta}^{-1/2}$ cancels in the component
probabilities. Tilting only by $Y$ preserves $K\mid Y$ and $X\mid Y,K$,
hence also $X\mid Y$, without importance resampling.

The quadratic working model uses prior variance $1.5$ and noise variance
$.94$. Logistic ratio fitting assigns total loss weight $1/2$ to each
domain, regardless of its sample size. Features $(y,y^2)$ are standardized
using training labels. The intercept is unpenalized, and the penalty is
$0.001/2$ times the squared standardized slopes. Coefficient constraints
are given in Section~\ref{sec:quadratic_exp}. All three estimators use
20 Gaussian predictive draws per input. Both TPS versions start at
zero, average successive parameter vectors with equal weights, and stop
at maximum coordinate change $10^{-6}$ or 100 iterations. TPS-Q-box uses
constrained fits at each iteration, not projection of the final
broad-range estimate. All final iterates are included.

At target-input sizes $(300,30,10)$, TPS-Q convergence rates are
$(100,96.5,82.0)\%$, against $(100,99.5,99.5)\%$ for TPS-Q-box.
The latter reaches some boundary in $(81.0,89.5,96.5)\%$ of fits and
the true corner in $(24.0,19.5,22.5)\%$, using tolerance $10^{-4}$.
TPS-Q reaches its upper cap in $(0,1.0,7.5)\%$. All evaluated Gaussian
tilts are proper.

Base seeds are 20260830 for real data, 20260927 for the linear study, and
20261031 for the quadratic study, with trial increments of 10007.
Real-data dataset offsets are 100000, with separate streams for
calibration/test draws and unlabeled target inputs. The linear bootstrap
uses offset 9000000 plus the condition index. Quadratic predictive draws
use offset 9000000 and share fixed normal variates across TPS versions,
nested target samples, and iterations. Paired sampling and standard-error
definitions are given in Section~\ref{sec:experiments}.

\end{document}